\documentclass{article}

\usepackage[preprint]{neurips_2026}
\usepackage{microtype}
\usepackage{graphicx}
\usepackage{subfigure}
\usepackage{tcolorbox}
\usepackage{booktabs} 
\usepackage{hyperref}
\usepackage[utf8]{inputenc} 
\usepackage[T1]{fontenc}    
\usepackage{url}            
\usepackage{booktabs}       
\usepackage{amsfonts}       
\usepackage{nicefrac}       
\usepackage{microtype}      
\usepackage{lipsum}
\graphicspath{{figs/}}     
\usepackage{algorithmic}
\usepackage{algorithm}%
\usepackage[title]{appendix}%
\usepackage{comment}
\usepackage{wrapfig}
\usepackage{xcolor}
\usepackage{amssymb}
\usepackage{amsmath}
\allowdisplaybreaks
\usepackage{wrapfig}
\usepackage{mathrsfs}
\usepackage{amsthm}
\usepackage{ifthen}
\newboolean{showcomments}
\setboolean{showcomments}{true}
\usepackage{color}
\usepackage{todonotes}

\definecolor{bleudefrance}{rgb}{0.19, 0.55, 0.91}
\definecolor{ao(english)}{rgb}{0.0, 0.5, 0.0}

\newcommand{\addcite}[0]{\ifthenelse{\boolean{showcomments}}
{\textcolor{purple}{(add cite(s)) }}{}}%

\newcommand{\addref}[0]{\ifthenelse{\boolean{showcomments}}
{\textcolor{purple}{(add ref) }}{}}%

\newcommand{\myparagraph}[1]{\vspace{.5 mm}\noindent\textbf{#1}}

\newcommand{\enrique}[1]{  \ifthenelse{\boolean{showcomments}}
{\todo[inline,color=bleudefrance]{Enrique: #1}}{}}
\newcommand{\rene}[1]{  \ifthenelse{\boolean{showcomments}}
{\todo[inline,color=cyan]{Ren\'e: #1}}{}}
\newcommand{\emmargin}[1]{\ifthenelse{\boolean{showcomments}}{\marginpar{\color{bleudefrance}\tiny EM: #1}}{}}
\newcommand{\hancheng}[1]{  \ifthenelse{\boolean{showcomments}}
{\todo[inline,color=orange]{Hancheng: #1}}{}}
\newcommand{\ziqing}[1]{  \ifthenelse{\boolean{showcomments}}
{\todo[inline,color=red]{Ziqing: #1}}{}}
\newcommand{\salma}[1]{  \ifthenelse{\boolean{showcomments}}
{\todo[inline,color=yellow]{Salma: #1}}{}}
\newcommand{\zxmargin}[1]{\ifthenelse{\boolean{showcomments}}{\marginpar{\color{purple}\tiny ZX: #1}}{}}
\newcommand{\stmargin}[1]{\ifthenelse{\boolean{showcomments}}{\marginpar{\color{red}\tiny ST: #1}}{}}

\newcommand{\hl}[1]{\ifthenelse{\boolean{showcomments}}
{\textcolor{red}{#1}}{#1}}

\newboolean{showedits}
\setboolean{showedits}{false}
\usepackage[markup=underlined, commandnameprefix=always]{changes}
\definechangesauthor[color=bleudefrance]{EM}
\newcommand{\aem}[1]{
\ifthenelse{\boolean{showedits}}
{\added[id=EM]{#1}}
{\!#1\hspace{-4.75pt}}
}
\newcommand{\repem}[2]{
\ifthenelse{\boolean{showedits}}
{\replaced[id=EM]{#1}{#2}}
{\!#1\hspace{-4.75pt}}
}
\newcommand{\dem}[1]{
\ifthenelse{\boolean{showedits}}
{\deleted[id=EM]{#1}}
{}
}

\usepackage{enumitem}
\setlist{leftmargin=*}

\newif \iffinal
\finalfalse
\iffinal

\newcommand{\WJmodified}[1]{{#1}}
\newcommand{\WJcomments}[1]{{}}
\newcommand{\oldText}[1]{}
\newcommand{\toAppendix}[1]{}
\else
\newcommand{\WJmodified}[1]{{\color{blue} #1}}
\newcommand{\WJcomments}[1]{{\WJmodified{Wenjie commented: #1}}}
\newcommand{\toAppendix}[1]{#1}
\newcommand{\oldText}[1]{#1}
\newcommand{\bx}{\mathbf{x}}
\newcommand{\bz}{\mathbf{z}}

\usepackage{mathtools}

\usepackage[capitalize,noabbrev]{cleveref}

\floatstyle{ruled}
\newfloat{algorithm}{tbp}{loa}

\floatname{algorithm}{Algorithm}
\theoremstyle{plain}
\newtheorem{thm}{Theorem}
\theoremstyle{plain}
\newtheorem{assumption}{Assumption}
\theoremstyle{plain}
\newtheorem{lem}{Lemma}
\theoremstyle{plain}
\newtheorem{rem}{Remark}
\theoremstyle{plain}
\newtheorem{dfn}{Definition}
\makeatother

\title{The Power of Memory: Fast Convex Optimization via Non-parametric Warm-up}

\title{Memory--Computation Tradeoffs in Semi Amortized Parametric Optimization}

\author{
Shijie Pan\textsuperscript{1} \qquad
Agustin Castellano\textsuperscript{1} \qquad
Zeyu Shen\textsuperscript{2} \qquad
Enrique Mallada\textsuperscript{1}\\[3pt]
\textsuperscript{1}Department of Electrical and Computer Engineering,
Johns Hopkins University\\
\textsuperscript{2}Department of Applied Mathematics and Statistics,
Johns Hopkins University\\
Baltimore, MD 21218, USA\\[3pt]
\texttt{\{span34, acaste11, mallada\}@jhu.edu},
\texttt{zshen39@jhu.edu}
}

\begin{document}

\maketitle

\begin{abstract}
Learning-enabled decision systems often use offline data or computation to reduce online compute cost. Despite the empirical success of such approaches, there is limited general understanding of how much offline information is needed to achieve a desired accuracy under a fixed online computation budget. We study this question through the lens of amortized parametric optimization: an offline phase stores a finite memory of solved problem instances, and an online phase produces a solution to a new instance by retrieving a warm start and applying $K$ steps of projected gradient descent. We analyze this setup for smooth convex parametric optimization over a compact domain, using a nonparametric predictor built from the stored offline solutions. For $\mu$-strongly convex objectives, we establish matching upper and lower bounds on the memory required to guarantee $\varepsilon$-accuracy under a fixed online iteration budget $K$. For convex objectives satisfying a $\beta$-growth condition ($\beta>2$), we obtain near-matching bounds and identify a phase transition in $K$ beyond which additional memory provides no benefit. We further provide a general proof framework that (i) explicitly quantifies the \emph{memory cost of acceleration}---how much offline memory is required to achieve a prescribed speedup over the unaided online optimizer---and (ii) identifies two key quantities driving this cost: the convergence rate of the online optimizer and the Lipschitz sensitivity of the solution map to the problem parameter. Experiments on parameterized ridge regression confirm the predicted memory--computation--accuracy tradeoffs.

\end{abstract}




\section{Introduction}\label{section1}

Modern machine learning systems increasingly require solving optimization problems under strict latency constraints. 
Across many applications today, decisions must be produced in real time, often allowing only a small and fixed computational budget for solution refinement. 
Examples include amortized variational inference in latent-variable models \citep{kingma2014auto,rezende2014stochastic}, 
learned sparse coding and algorithm unrolling \citep{gregor2010learning}, 
policy and value function approximation in reinforcement learning \citep{sutton2018reinforcement}, 
and model predictive control and real-time trajectory optimization \citep{mayne2014model,jerez2014embedded}. 
A common pattern in these systems is to shift most of the computational burden to an offline phase, 
where a predictor $\pi_\mathcal{D}$ is trained or constructed, and then used at test time to produce a fast approximate solution, 
often followed by a small number of refinement steps.

This paradigm is broadly referred to as \emph{amortized optimization} \citep{amos2022tutorial}. 
In this setting, past computation over a family of related problems is leveraged to accelerate the solution of future instances. 
Depending on the setting, the predictor may directly approximate the solution map $\bx^\star(\theta)$ (fully amortized), 
or provide an initialization that is refined by a small number $K$ of optimization steps (semi-amortized) \citep{kim2018semi,andrychowicz2016learning}. 
Despite strong empirical success, the theoretical understanding of amortized optimization remains limited. 
As emphasized in recent surveys \citep{amos2022tutorial,chen2022learningtooptimize}, existing analyses primarily focus on generalization or stability of the learned predictor, 
and do not directly address the role of computational constraints at inference time. 
In particular, there is currently no general theory that quantifies how much offline information is required to guarantee a prescribed level of accuracy under a fixed inference time computation budget.

This paper takes a first step towards closing this gap by studying the data--computation trade-offs of amortized parametric optimization, focusing on the following question:
\begin{center}
\emph{Given a fixed online computation budget of $K$ iterations, how much offline information is required to guarantee $\varepsilon$-accurate solutions for all problem instances?}
\end{center}
{This question is, in essence, the basic framing of semi-amortized optimization~\citep{amos2022tutorial}, and reduces to fully amortized optimization when $K=0$.}

To answer this question concretely, we adopt a fixed instantiation: a parametric family of convex problems, projected gradient descent (PGD) as the online refinement, and a nonparametric predictor that picks the best stored solution as the warm start.
Based on this semi-amortized optimization paradigm, we develop a theoretical framework that characterizes the \emph{memory--computation--accuracy tradeoff} in semi amortized parameteric convex optimization, identifying problem structure as the key determinant of how effectively memory substitutes for online computation.
\begin{table*}[htb]
\centering
\caption{Memory complexity $M(K,\varepsilon)$ tradeoff under different function classes.}
\label{tab:complexity_summary}
\begin{tabular}{lcc}
\toprule
Function class & $\mathcal{O}(\,\cdot\,)$ upper bound & $\Omega(\,\cdot\,)$ lower bound \\
\midrule
$\mu$-strongly convex smooth
 & $\left(\dfrac{\rho^K}{\varepsilon}\right)^{\!d_\Theta}$
 & $\left(\dfrac{\rho^K}{\sqrt{\varepsilon}}\right)^{\!d_\Theta}$ \\[6pt]
Convex smooth + $\beta$-growth ($\beta>2$)
 & $\left[\varepsilon^{-\frac{\beta-2}{\beta}}-\Omega(K)\right]_+^{\frac{\beta\, d_\Theta}{\beta-2}}$
 & $\left[\varepsilon^{-\frac{\beta-2}{\beta}}-\mathcal{O}(K)\right]_+^{\frac{d_\Theta}{\beta-2}}$ \\
\bottomrule
\end{tabular}
\end{table*}

\myparagraph{Paper contributions.} Concretely, our work provides the following contributions:
\begin{itemize}
    \item \textit{Theoretical framework.} We formalize the tradeoff through two dual resource quantities: $M(K,\varepsilon)$, the minimum memory required to guarantee worst-case accuracy $\varepsilon$ over the parameter space under $K$ PGD steps; and $M_\alpha(K)$, the minimum memory required to achieve a prescribed acceleration $\alpha$ over the unaided online optimizer. Together, $M(K,\varepsilon)$ and $M_\alpha(K)$ formalize the \emph{memory cost of acceleration}---a quantitative assessment of the effectiveness of semi amortized optimization.
    
\item \textit{Strongly convex regime (Theorem~\ref{thm1}; Corollary ~\ref{coro1})}
    For $\mu$-strongly convex smooth objectives, we obtain matching upper and lower bounds on $M(K,\varepsilon)$ that differ only by a factor of two in the exponent of the accuracy term $\varepsilon$ (first row in Table \ref{tab:complexity_summary}). The corresponding bound on $M_\alpha(K)$ reveals a striking finding: even modest speedups over the already-linearly-convergent baseline require memory that grows \emph{exponentially} in $K$. Thus, linearly-convergent methods are intrinsically hard to accelerate further with memory.
    
    \item \textit{$\beta$-growth regime (Theorem~\ref{thm2}; Corollary \ref{coro3})}
    For convex objectives satisfying a $\beta$-growth condition with $\beta>2$, we establish near-matching upper and lower bounds whose constants agree up to factors independent of $K$ and $\varepsilon$ (second row in Table \ref{tab:complexity_summary}). The bounds expose an intrinsic \emph{phase transition} in $K$: beyond a problem-dependent threshold, additional memory provides no further benefit. Since this transition appears in both the upper and the lower bound, it is an intrinsic property of the problem itself rather than of the analysis. As
    $\beta\to 2^+$, both bounds recover Theorem~\ref{thm1}
    (Proposition~\ref{prop:beta_to_2}).
    
    \item \textit{Meta proof framework.} The two regime results share a common skeleton, which we distill into a single \emph{meta-theorem}: a reduction of the analysis to two structural inputs---the convergence rate of the online optimizer and the Lipschitz sensitivity of the solution map to the parameter. This isolates what governs amortization efficiency and provides a uniform recipe for extending the framework to other optimization classes, including accelerated methods, projection-free schemes, and problems with implicit constraints.
    
    \item \textit{Insights and empirical validation.} On parametrized ridge regression---a setting where both structural drivers admit closed-form expressions---experiments confirm the theory in three concrete ways: (i) the empirical scaling of $M(K,\varepsilon)$ matches the strongly-convex bound from Theorem~\ref{thm1}; (ii) the $\beta$-growth bound traces the observed memory--accuracy curve, including the predicted phase-transition threshold; and (iii) the dependence on the convergence rate and parametric sensitivity matches the closed-form predictions, validating the meta-theorem's two-input
    characterization.
\end{itemize}

This paper is organized as follows. Section~\ref{relwork} reviews related literature. Section \ref{problemsetup} formalizes the parametric problem and the offline--online hybrid regime. Section~\ref{abserrresult} establishes our main bounds on $M(K,\varepsilon)$ for $\mu$-strongly convex (Theorem~\ref{thm1}) and $\beta$-growth (Theorem~\ref{thm2}) objectives. Section \ref{accel} reinterprets these bounds through the memory cost of acceleration; Section \ref{meta proof} presents the meta proof framework that unifies both. Section~\ref{experiment} validates the predictions on parametrized ridge regression, and Section \ref{conclusions} concludes.

\section{Related Literature}\label{relwork}
Our work sits at the intersection of three lines of research. We organize the discussion accordingly and clarify how each relates to the memory--accuracy--computation tradeoff studied here.

\myparagraph{Amortized optimization and learning to optimize.} Amortized optimization \citep{amos2022tutorial} and learning to optimize \citep{chen2022learningtooptimize} use machine learning to predict the solutions of repeatedly-encountered parametric optimization problems, exploiting shared structure across instances. The model can be \emph{fully amortized}, directly mapping a problem to its solution---as in variational autoencoders \citep{kingma2014auto,rezende2014stochastic}, sparse coding via algorithm unrolling \citep{gregor2010learning}, differentiable optimization layers \citep{amos2017optnet,agrawal2019differentiable}, and policy networks for control and reinforcement learning \citep{sutton2018reinforcement}---or \emph{semi-amortized}, providing an initialization that is refined by a small number of optimization steps \citep{kim2018semi,andrychowicz2016learning,venkataraman2021neural}. This literature is dominated by methodology and empirical results; theoretical analyses primarily address generalization or stability of the learned model, and the open problem of complexity-style guarantees for the underlying resource tradeoff is acknowledged in recent surveys \citep{amos2022tutorial,chen2022learningtooptimize}. Our work contributes to this gap by characterizing the memory--computation--accuracy tradeoff intrinsic to the problem class, independently of which predictor instantiates it.

\myparagraph{Predictor--refinement methods: warm-starting and meta-learning.} The methodologically closest neighbours adopt the same predictor-then-refine template that we study. For convex programming, learned warm starts have been investigated for QPs and MIQPs through dual-variable initialization \citep{sugishita2024use}, binary-variable fixing \citep{xavier2021learning,bertsimas2022online}, branch-and-bound seeding \citep{marcucci2020warm}, and active-set prediction with neural and graph networks \citep{klauvco2019machine,schmidtobreick2025warm,arnstrom2021semi}. In model predictive control, warm-starts predict primal solutions used directly without further refinement \citep{zhang2019safe,hertneck2018learning,tokmak2025automatic}. Most relevant to our setting, \citep{sambharya2023end,sambharya2024learning} cast the optimization process as a fixed-point iteration and learn warm-start initializations end-to-end, providing generalization bounds on a fixed number of online iterations. A complementary thread in gradient-based meta-learning---initiated by \citep{finn2017model} and extended through clustering \citep{vuorio2019multimodal,peng2023clustered}, kernel \citep{wang2020structured}, hypernetwork \citep{zhao2020meta}, Bayesian \citep{finn2018probabilistic}, and direct-prediction \citep{requeima2019fast} variants---learns shared task representations enabling rapid fine-tuning. Our contribution is methodologically and analytically orthogonal: rather than studying generalization of a learned predictor across training tasks, we ask the inverse, resource-allocation question---how many stored solutions are required at inference time to guarantee $\varepsilon$-accuracy under a $K$-iteration budget---and answer it with matching upper and lower bounds on $M(K,\varepsilon)$ that hold for any predictor of the form considered.

\myparagraph{Theoretical perspectives on offline data and online computation.} A separate theoretical tradition asks how offline information can accelerate online computation in a complexity-theoretic sense. Algorithms with predictions \citep{mitzenmacher2020algorithms,lykouris2018competitive,purohit2018improving} augment classical online algorithms with a learned signal whose error gracefully degrades the competitive ratio; subsequent work studies how to learn these predictions themselves \citep{khodak2022learningpredictions}. Data-driven algorithm design \citep{gupta2020data,balcan2020data} characterizes the sample complexity of selecting algorithm parameters from a distribution of problem instances. Our work shares this complexity-theoretic spirit but transports it to a different setting: continuous parametric convex optimization, where the offline resource is a memory of stored solutions rather than predictions of future requests, and the online resource is a bounded number of optimization iterations rather than an arrival sequence. To our knowledge, no prior work in this tradition provides matching memory bounds for parametric convex optimization under a fixed online iteration budget.

Across all three lines, the basic question---\emph{given an online budget of $K$ iterations, how much offline information is required to guarantee $\varepsilon$-accuracy over the parameter space?}---has not been answered with matching upper and lower bounds for parametric convex optimization. Characterizing this tradeoff, and identifying the two structural inputs (convergence rate and parametric sensitivity) that determine it, is the contribution of the present paper.

\section{Problem Setup}
\label{problemsetup}

\subsection{Parameterized Convex Optimization Problem}

We consider a parameterized convex optimization problem of the form
\begin{align}
f^\star(\theta) := \min_{\bx \in \mathcal{X}} \, f(\bx; \theta),
\label{cpp}
\end{align}
where $\theta \in \Theta$ denotes a parameter that specifies a member of a family of convex optimization problems, $\bx \in \mathcal{X}$ is the decision variable, and $f^\star(\theta)$ denotes the optimal value associated with parameter $\theta$. We denote by
\[
\bx^\star(\theta) \in \arg\min_{\bx \in \mathcal{X}} f(\bx;\theta)
\]
an optimal solution of \eqref{cpp}. Throughout the paper, we study how the optimizer and optimal value vary as a function of the parameter $\theta$.

We further make the following assumptions for problem \eqref{cpp}:

\begin{assumption} \label{a1}
For any given $\theta\in\Theta$, $f(\bx;\theta):\mathcal{X}\times\Theta\to\mathbb{R}$ is a convex function in \(\bx \in \mathcal{X} \).
\end{assumption}
\begin{assumption}\label{a2}
The decision space \( \mathcal{X} \subset \mathbb{R}^{d_{\mathcal{X}}} \) is compact with positive Lebesgue measure and convex with radius \( R_{\mathcal{X}} \) under an arbitrary norm \( \|\cdot\| \). The parameter space \( \Theta \subset \mathbb{R}^{d_\Theta} \) is compact with radius \( R_\Theta \) under the same norm $\|\cdot\|$. We assume that \( d_{\mathcal{X}} \gg d_\Theta \).
\end{assumption}
\begin{assumption}\label{a3}
For any fixed \( \bx \in \mathcal{X} \), the function \( f(\bx;\theta) \) is \( L_\Theta \)-Lipschitz in \( \theta \) under \( \|\cdot\| \). For any fixed \( \theta \in \Theta \), the function \( f(\bx; \theta) \) is \( L_{f,1} \)-Lipschitz under \( \|\cdot\| \) and \( L_{f,2} \)-smooth under \( \|\cdot\|_2 \).
\end{assumption}
\begin{assumption}\label{a4}
For all \( \theta \in \Theta \), the feasible set is fixed and given by \( \mathcal{X} \).
\end{assumption}

\myparagraph{Assumption justification.}
Assumptions~\ref{a1}--\ref{a4} are basic assumptions that lay the foundation for all later technical results. The convexity in Assumption~\ref{a1} is invoked only in the derivation of the convergence rate of the online optimizer; since our meta proof framework treats this rate as a black-box input, the analysis would extend to non-convex settings whenever a comparable rate is available. The compactness of the parameter space in Assumption~\ref{a2} is critical for the memory bounds: the parameter dimension $d_\Theta$ controls the exponent of $M(K,\varepsilon)$ and the radius $R_\Theta$ enters the prefactor through the covering of $\Theta$. The Lipschitz and smoothness conditions in Assumption~\ref{a3} are standard in the amortized optimization literature \citep{amos2022tutorial}. In particular, the smoothness constant $L_{f,2}$ is defined with respect to the $\|\cdot\|_2$ norm, consistent with standard smoothness definition. Assumption~\ref{a4} fixes the feasible set across parameters for concreteness; extending the analysis to parameter-dependent constraints is subject to future work.

\subsection{Semi-amortized optimization framework}
Given the parameterized convex optimization problem in \eqref{cpp}, our goal is to design an offline--online hybrid strategy that achieves an $\varepsilon$-optimal solution within at most $K$ online refinement steps. The strategy is composed of two phases, illustrated in Figure~\ref{fig:framework-overview}. At each phase, the strategy has access to an oracle that provides a different type of information about \eqref{cpp}.
\begin{itemize}\vspace{-1ex}
    \item \textit{Offline oracle:} Given parameter $\theta \in \Theta$, output an optimal solution $\bx^\star(\theta)$ of \eqref{cpp}.
    \item \textit{Online oracle:} Given decision variable $\bx \in \mathcal{X}$ and parameter $\theta \in \Theta$, output gradient $\nabla_\bx f(\bx;\theta)$.\vspace{-1ex}
\end{itemize}
\begin{figure}[!t]
    \centering
    \includegraphics[width=0.7\linewidth]{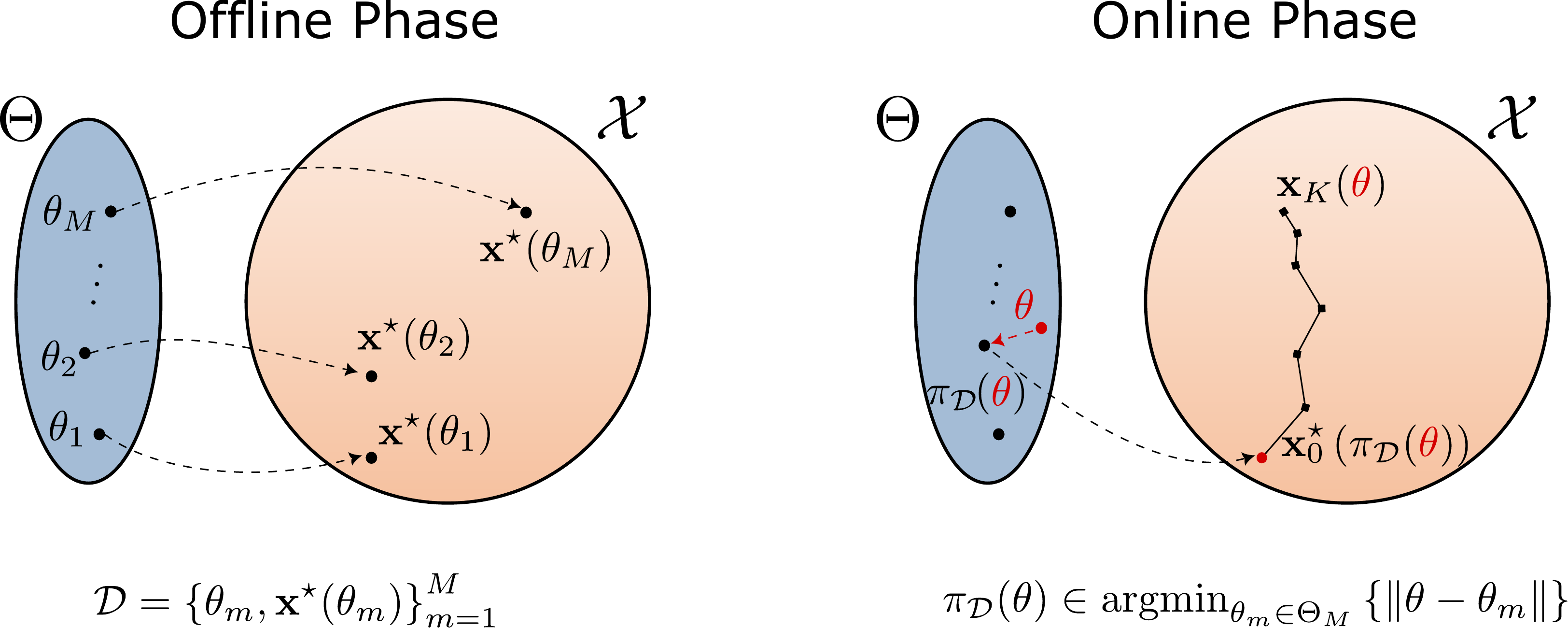}
    \caption{An overview of our framework. \emph{Left}: During the offline phase we sample parameters $\{\theta_m\}_{m=1}^M$ and store their optimal solutions in the dataset $\mathcal{D}$. \emph{Right}: Given a test parameter $\theta$, the online phase warm starts $\bx_0(\theta)$ with the optimal solution for the predictor $\pi_\mathcal{D}(\theta)$. This is followed by $K$ steps of projected gradient descent.}
    \label{fig:framework-overview}
\end{figure}
\myparagraph{Offline phase.} A finite point set $\Theta_M = \{\theta_m\}_{m=1}^M \subset \Theta$ is sampled in the parameter space. For each $\theta_m \in \Theta_M$, a single offline-oracle call returns the optimal solution $\bx^\star(\theta_m)$, yielding the dataset $\mathcal{D}=\{(\theta_m,\bx^\star(\theta_m))\}_{m=1}^M$ of $M$ stored solutions (memory). From $\mathcal{D}$, we construct a nonparametric predictor via the nearest-neighbor rule:
\begin{align}\label{1NN}
\pi_{\mathcal{D}}:\Theta \to \Theta_M,\quad
\pi_{\mathcal{D}}(\theta)
\in
\arg\min_{\theta_m \in \Theta_M}
\left\{ \|\theta - \theta_m\|
\right\}.
\end{align}
\myparagraph{Online phase.} Given a test parameter $\theta \in \Theta$, the predictor returns a warm start $\bx_0(\theta) := \bx^\star\!\bigl(\pi_{\mathcal{D}}(\theta)\bigr)$ retrieved from $\mathcal{D}$. Starting from $\bx_0(\theta)$, we run $K$ steps of projected gradient descent (PGD) \citep[Section~1.4.3]{beck2010gradient}, 
\begin{align}\label{pgd}
\bx_{k+1}(\theta)=\Pi_{\mathcal{X}}\!\left(\bx_k(\theta)-\eta \nabla_{\bx} f(\bx_k(\theta);\theta)\right),\quad k=0,1,\ldots,K-1.
\end{align}
with constant stepsize $\eta=L_{f,2}^{-1}$, thus incurring a total of $K$ online oracle calls.

\myparagraph{Ultimate goal.}~The ultimate goal of the hybrid strategy is thus to guarantee, for all $\theta\in\Theta$, an $\varepsilon$-optimal solution within at most $K$ online oracle calls:
\begin{align}\label{epsilonoptimal}
\min_{k \in [K]} f(\bx_k(\theta);\theta) - f(\bx^\star(\theta);\theta) \le \varepsilon
\end{align}
The detailed implementation of the strategy is described in Algorithm \ref{alg:data_collector} and Algorithm \ref{alg:data_collector1}.
\begin{figure*}[!h]\vspace{0ex}
\centering
\small
\begin{minipage}[t]{0.48\linewidth}
\begin{algorithm}[H]
\caption{Offline uniform grid sampler}
\label{alg:data_collector}
\begin{tcolorbox}[left=0mm,right=0mm,top=0mm,bottom=0mm]{
\begin{algorithmic}
\REQUIRE Sampling budget $M$.
\STATE \textbf{Initialize:} $\mathcal{D}=\emptyset$.
\STATE Uniform grid $\Theta_{M,\mathrm{Uni}} \subset \Theta$ with $M$ points.
\FOR{$m=1,\ldots,M$}
    \STATE Sample at grid point $\theta_m \in \Theta_{M,\mathrm{Uni}}$.
    \STATE Solve $\bx^\star(\theta_m)\in\arg\min_{\bx\in\mathcal{X}}f(\bx;\theta_m)$.
    \STATE Add $(\theta_m,\bx^\star(\theta_m))$ to $\mathcal{D}$.
\ENDFOR
\STATE \textbf{Output:} Non-parametric predictor:\STATE $\pi_{\mathcal{D}}(\cdot)\in\arg\min_{\theta_m\in\Theta_M} \left\{\|\cdot-\theta_m\|\right\}$.
\end{algorithmic}}
\end{tcolorbox}
\end{algorithm}
\end{minipage}
\hfill
\begin{minipage}[t]{0.48\linewidth}
\begin{algorithm}[H]
\caption{Online refinement algorithm}
\label{alg:data_collector1}
\begin{tcolorbox}[left=0mm,right=0mm,top=0mm,bottom=0mm]{
{\fontsize{9.5}{12.5}\selectfont{\begin{algorithmic}
\REQUIRE Online budget $K$, predictor $\pi_{\mathcal{D}}$.
\STATE \textbf{Input:} $\theta\in\Theta$. $\mathcal{T}=\emptyset$.
\STATE Apply $\theta_{m^\star}=\pi_{\mathcal{D}}(\theta)$, take $\bx_0\gets \bx^\star(\theta_{m^\star})$.
\FOR{$k=0,\ldots,K-1$}
    \STATE $\bx_{k+1}\gets\Pi_{\mathcal{X}}(\bx_k-L_{f,2}^{-1}\nabla_\bx f(\bx_k;\theta))$.
    \STATE Append $(\bx_{k+1},f(\bx_{k+1};\theta))$ to $\mathcal{T}$.
\ENDFOR
\STATE  $\bx^{\star\star}\gets\arg\min_{(\bx_{k},f(\bx_{k};\theta))\in\mathcal{T}}f(\bx_{k};\theta)$.
\STATE \textbf{Output:} $\bx^{\star\star}$.
\end{algorithmic}}}}
\end{tcolorbox}
\end{algorithm}
\end{minipage}
\end{figure*}

Because the warm-start error at any test parameter $\theta$ is controlled by the distance $\|\theta - \pi_\mathcal{D}(\theta)\|$ to its nearest stored sample---by the Lipschitz continuity of $f$ in Assumption~\ref{a3}---bounding the worst-case error over $\Theta$ amounts to ensuring that $\Theta_M$ covers $\Theta$ at a sufficiently fine resolution. 

We formalize this requirement, and the resulting minimal memory, as follows.
\begin{dfn}[$(K,\varepsilon)$-$\Theta$-net]\label{def:net}
A point set $\Theta_M \subset \Theta$ is a $(K,\varepsilon)$-$\Theta$-net if the warm-start predictor $\pi_\mathcal{D}$ built from $\Theta_M$ guarantees~\eqref{epsilonoptimal} for every $\theta \in \Theta$: an $\varepsilon$-optimal solution is obtained within at most $K$ online refinement steps.
\end{dfn}
A $(K,\varepsilon)$-$\Theta$-net can be viewed as a covering of $\Theta$ whose admissible radius is determined by $K$ and $\varepsilon$. Clearly, any such $(K,\varepsilon)$-$\Theta$-net would provide a valid solution to our goal \eqref{epsilonoptimal}, thus the minimum cardinality of such set constitutes a natural definition of memory.
\begin{dfn}[Memory complexity $M(K,\varepsilon)$]\label{def:memory}
The \emph{memory complexity} at online budget $K$ and accuracy $\varepsilon$ is the smallest number of stored solutions required to form a $(K,\varepsilon)$-$\Theta$-net:
\[
M(K,\varepsilon) \;:=\; \min\bigl\{\,M\,:\, \exists\, \Theta_M \subset \Theta \text{ that is a } (K,\varepsilon)\text{-}\Theta\text{-net}\,\bigr\}.
\]
\end{dfn}
Note that $M(K,\varepsilon)$ depends on the warm-start predictor $\pi_{\mathcal{D}}$ and on the online refinement procedure; we omit this dependence for notational brevity.
\begin{rem}[Uniform random sampling vs.\ grid sampling] As a practical alternative to the uniform grid sampler in Algorithm~\ref{alg:data_collector}, one can replace it with a uniform random sampler over $\Theta$ (Algorithm~\ref{alg:data_collector2} in Appendix~\ref{appendixA}). 
The two achieve the same memory--accuracy guarantees up to log factors, with the random sampler additionally requiring a regularity condition on the geometry of $\Theta$ (Assumption \ref{a10}); see Appendix~\ref{appendixA}.
\end{rem}
\begin{rem}[Uniform grid as a cover]
In Algorithm~\ref{alg:data_collector}, the uniform grid $\Theta_{M,\mathrm{Uni}}$ induces a cover under the $\|\cdot\|_\infty$ norm. By norm equivalence in finite-dimensional spaces (see Theorem~2.1 in \cite{jin2016introduction}), the covering numbers of $\Theta$ under $\|\cdot\|_\infty$ and $\|\cdot\|$ are of the same order.
\end{rem}
\newtheorem{prop}{Proposition}
\section{Main results: memory--accuracy--computation tradeoff}\label{abserrresult}
In this section we establish upper and lower bounds on the memory complexity $M(K,\varepsilon)$ as a function of the online refinement budget $K$ and prescribed error tolerance $\varepsilon$, under different structural assumptions on $f$. Section~\ref{sec:strongly-convex} deals with the case of $f$ being strongly convex, while Section~\ref{gconvex} deals with functions $f$ satisfying the $\beta$-growth condition. Finally, we show in Section~\ref{accel} how these bounds can be equivalently expressed through the dual quantity $M_\alpha(K)$, the memory cost of achieving a prescribed acceleration $\alpha$ over the unaided online optimizer.

\subsection{Strongly convex case} \label{sec:strongly-convex}
We first consider the strongly convex regime, which serves as our baseline.
\begin{assumption}\label{a6}
For every $\theta \in \Theta$, $f(\,\cdot\,;\theta)$ is $\mu$-strongly convex over $\mathcal{X}$:
\[
f(\bx_1;\theta) \;\ge\; f(\bx_2;\theta) + \nabla_{\bx} f(\bx_2;\theta)^{\!\top}(\bx_1 - \bx_2) + \tfrac{\mu}{2}\|\bx_1 - \bx_2\|_2^2,\quad \forall\,\bx_1,\bx_2 \in \mathcal{X}.
\]
\end{assumption}
Under Assumption~\ref{a6}, PGD converges linearly at rate $\rho := 1 - L_{f,2}^{-1}\mu$. The following theorem (proved in Appendix~\ref{appendix1}) characterizes how much offline memory is required to guarantee $\varepsilon$-accuracy under a $K$-step online budget.
\begin{thm}\label{thm1}
Under Assumptions~\ref{a1}--\ref{a6}, given an online refinement budget of $K$ PGD steps and an accuracy target $\varepsilon > 0$, the memory complexity $M(K,\varepsilon)$ admits upper and lower bounds:
\[
\Omega\!\left(\!\left(R_\Theta L_\Theta\,{\tfrac{\rho^K}{\sqrt{\varepsilon}}}\right)^{\!d_\Theta}\!\right)
\;\le\; M(K,\varepsilon) \;\le\;
\mathcal{O}\!\left(\!\left(R_{\Theta}L_\Theta\,\tfrac{\rho^K}{\varepsilon}\right)^{\!d_\Theta}\right),
\]
with $\rho := 1 - L_{f,2}^{-1}\mu$.
\end{thm}

\myparagraph{Almost matching bounds.}
The two bounds differ only by a factor of two in the exponent of the error $\varepsilon$ term: $(\rho^K/\varepsilon)^{d_\Theta}$ in the upper bound versus $(\rho^K/\sqrt{\varepsilon})^{d_\Theta}$ in the lower bound. The dependence on the parameter variation $L_\Theta R_\Theta$, $\rho^K$ and on the parameter dimension $d_\Theta$ is identical in both, so the qualitative scaling is tight; only the constant inside the accuracy exponent is open.

\myparagraph{Drivers of memory complexity.}
The bound makes three structural quantities explicit: the accuracy $\varepsilon$, the convergence rate $\rho$ (entering as $\rho^K$, the per-step contraction PGD already provides), and the parameter variation $L_\Theta R_\Theta$ (how much the value of $f(\bx,\theta)$ varies across $\Theta$). All three enter inside the $d_\Theta$-th power. Crucially, the decision dimension $d_\mathcal{X}$ does \emph{not} appear---memory scales with the parameter manifold, not with the ambient solution space, which is what makes amortization viable in the $d_\mathcal{X} \gg d_\Theta$ regime of Assumption~\ref{a2}.

\myparagraph{Phase transition.}
Once $K \ge {\log(1/\varepsilon)}/{\log(1/\rho)}$, the upper bound becomes vacuous: with $L_\Theta R_\Theta = \Theta(1)$, the memory complexity collapses to $M(K,\varepsilon) = \Theta(1)$. Beyond this threshold, $\varepsilon$-accuracy is reached by PGD alone, and offline memory provides no further benefit.

\subsection{Convex with $\beta$-growth case}\label{gconvex}
In order to characterize the memory complexity in setting where PGD lacks linear convergence,  we now consider a broader class of convex problems satisfying a $\beta$-growth condition.
\begin{assumption}[$\beta$-growth, \cite{xu2017stochastic}]\label{a7}
 There exists $\Phi > 0$ such that
\[
f(\bx;\theta) - f^\star(\theta) \;\ge\; \Phi\,\mathrm{dist}_2^\beta\!\bigl(\bx,\,\arg\min_{\bx \in \mathcal{X}} f(\bx;\theta)\bigr),\quad \forall\,\bx \in \mathcal{X},\, \forall\,\theta \in \Theta.
\]
\end{assumption}
\myparagraph{Convergence rates under $\beta$-growth.}
The exponent $\beta$ controls how the objective behaves near its minimizer and, consequently, the convergence rate of PGD. Classical results \citep{li2018calculus,frankel2015splitting} distinguish three regimes: PGD converges in finitely many steps when $\beta=1$, linearly when $\beta\in(1,2]$, and only sublinearly when $\beta>2$. The strongly convex case treated in Theorem~\ref{thm1} corresponds to $\beta = 2$ and lies within the linear regime. 

We focus on the sublinear regime $\beta>2$, which lies outside the scope of Theorem~\ref{thm1} and exhibits a qualitatively different memory--computation tradeoff.

\begin{thm}[Memory complexity under $\beta$-growth]\label{thm2}
Suppose Assumptions~\ref{a1}--\ref{a4} and \ref{a7} hold with $\beta > 2$. 
\begin{flalign*}
&\text{Suppose further that:} \qquad
L_\Theta \le L_{f,2}
\qquad \text{and} \qquad
\Phi \le \Phi_{\mathrm{lb}}
:= \tfrac{L_\Theta}
{\beta(\beta-1)(R_\mathcal{X}+R_\Theta)^{\beta-2}}. &
\end{flalign*}
Then the memory complexity admits upper and lower bounds:
\[
\Omega\!\left(\!(L_\Theta R_\Theta)^{\frac{d_\Theta}{\beta}}\,\bigl[\,\varepsilon^{-\frac{\beta-2}{\beta}} \!- \mathcal{O}(K)\,\bigr]_+^{\frac{d_\Theta}{\beta-2}}\right) \;\le\; M(K,\varepsilon) \;\le\; \mathcal{O}\!\left(\!(L_\Theta R_\Theta)^{d_\Theta}\,\bigl[\,\varepsilon^{-\frac{\beta-2}{\beta}} \!- \Omega(K)\,\bigr]_+^{\frac{\beta d_\Theta}{\beta-2}}\right),
\]
where $[\cdot]_+:=\max\{0,\cdot\}$.
\end{thm}

\myparagraph{Parameter restrictions.}
Both auxiliary assumptions in Theorem~\ref{thm2} concern only the lower bound. The first, $L_\Theta \le L_{f,2}$, is a presentational simplification of the lower-bound prefactor $(L_\Theta R_\Theta)^{{d_\Theta}/{\beta}}$; the full statement without it appears in Appendix~\ref{pfthm2}. The second, $\Phi \le \Phi_{\mathrm{lb}}$, is mild: within our lower bound class construction, smoothness alone forces $\Phi$ to be of the same order, so the lower bound holds essentially throughout the admissible range of $\Phi$.

\myparagraph{Phase transition in $K$.}
The argument inside the bracket $[\,\cdot\,]_+$ vanishes once $K$ exceeds a threshold of order $\varepsilon^{-(\beta-2)/\beta}$, beyond which the upper bound becomes vacuous and offline memory provides no further benefit. This threshold matches the sublinear convergence rate of PGD under $\beta$-growth, $\mathcal{O}(K^{-\beta/(\beta-2)})$ \citep{li2018calculus}: the transition is precisely the regime where PGD alone already reaches $\varepsilon$-accuracy, so memory has nothing to amortize. 

\myparagraph{Comparison with the strongly convex case.}
The upper and lower bounds in Theorem~\ref{thm2} differ by a factor of $\beta$ in the exponent --- $\beta d_\Theta/(\beta-2)$ in the upper bound versus $d_\Theta/(\beta-2)$ in the lower. As $\beta \to 2^+$, this ratio matches the factor-of-two exponent gap of Theorem~\ref{thm1}, and both bounds recover Theorem~\ref{thm1} in the $\mathcal{O}$ and $\Omega$ senses (Appendix~\ref{pfprop1}). Away from this boundary the contrast is stark: under strong convexity, memory must grow \emph{exponentially} in $K$ to extract any further acceleration over PGD's already-fast linear baseline; under $\beta$-growth, polynomial growth suffices, since the sublinear baseline leaves more room for memory to help. Section~\ref{accel} formalizes this contrast through the memory cost of acceleration $M_\alpha(K)$.




\myparagraph{Convex smooth extension.}
Without any growth condition, an upper bound on $M(K,\varepsilon)$ analogous to Theorem~\ref{thm2} still holds (Theorem~\ref{thm3}, Appendix~\ref{appendix2}); matching lower bounds in this setting are more subtle and we leave them open. The local $\beta$-growth case, where Assumption~\ref{a7} holds only in a neighborhood of $\bx^\star(\theta)$, is also subsumed provided iterates remain in the region.

\subsection{Memory-based acceleration analysis}\label{accel}
We now reinterpret Theorems~\ref{thm1} and \ref{thm2} through a dual lens: rather than fixing accuracy $\varepsilon$ and characterizing the memory $M(K,\varepsilon)$ required, we fix an \emph{acceleration factor} $\alpha \in (0,1)$ and characterize the memory $M_\alpha(K)$ needed to drive PGD beyond its native convergence rate. Smaller $\alpha$ corresponds to a more aggressive target. The quantity $M_\alpha(K)$ is the \emph{memory cost of acceleration}: how much offline memory must be invested to obtain a prescribed speedup over the unaided online optimizer.

\myparagraph{Strongly convex case.}
Under the conditions of Theorem~\ref{thm1}, PGD converges at rate $\rho^K$. The acceleration factor $\alpha$ multiplies this rate to yield a more aggressive target $\mathcal{O}\bigl((\alpha\rho)^K\bigr)$, and Theorem~\ref{thm1} specializes to:

\newtheorem{coro}{Corollary}
\begin{coro}[Acceleration in the strongly convex case]\label{coro1}
Under the conditions of Theorem~\ref{thm1}, to achieve a target convergence rate $\mathcal{O}\bigl((\alpha\rho)^K\bigr)$ with $\alpha\in(0,1)$, it suffices that
\[
M_\alpha(K) \;\le\; \mathcal{O}\!\left(\!\left(R_\Theta L_\Theta\,(1/\alpha)^K\right)^{\!d_\Theta}\right).
\]
\end{coro}

\myparagraph{Convex with $\beta$-growth case.}
Under the conditions of Theorem~\ref{thm2} with $\beta > 2$, PGD's native rate is $\mathcal{O}\bigl(K^{-\beta/(\beta-2)}\bigr)$. Here the acceleration factor $\alpha$ divides the polynomial exponent, yielding the more aggressive target $\mathcal{O}\bigl(K^{-\beta/[\alpha(\beta-2)]}\bigr)$, and Theorem~\ref{thm2} specializes to:
\begin{coro}[Acceleration in the $\beta$-growth case]\label{coro3}
Under the conditions of Theorem~\ref{thm2} with $\beta > 2$, to achieve a target convergence rate $\mathcal{O}\bigl(K^{-\beta/[\alpha(\beta-2)]}\bigr)$ with $\alpha\in(0,1)$, it suffices that
\[
M_\alpha(K) \;\le\; \mathcal{O}\!\left(\!\left(R_\Theta L_\Theta\,K^{\beta/[\alpha(\beta-2)]}\right)^{\!d_\Theta}\right).
\]
\end{coro}
The proofs of Corollaries~\ref{coro1} and \ref{coro3} appear in Appendices~\ref{pfcoro1} and \ref{pfcoro3}.

\myparagraph{Memory cost of acceleration.}
The two corollaries reveal a sharp asymmetry in how the memory cost of acceleration scales with $K$. In the strongly convex case, $M_\alpha(K)$ grows \emph{exponentially} in $K$: accelerating an already-fast linearly convergent method demands an exponential investment in offline memory. In the $\beta$-growth case, $M_\alpha(K)$ grows only \emph{polynomially} in $K$: slower online baselines leave more room for offline memory to help.

\myparagraph{Limit $\beta \to 2^+$.}
As $\beta \to 2^+$, the polynomial scaling of Corollary~\ref{coro3} ceases to deliver effective acceleration: the exponent $\beta/[\alpha(\beta-2)]$ diverges, so $M_\alpha(K)$ grows faster than any polynomial in $K$. Under the natural correspondence between $\Phi$ and $\rho$ (Lemma~\ref{lem2}, Appendix~\ref{pfspeed}), this limit smoothly matches the exponential scaling of Corollary~\ref{coro1}, providing consistency between the two regimes.

\smallskip\noindent
Together, the two corollaries reveal a substitution principle:
\begin{center}
\textbf{\textit{Amortization helps most where online computation is expensive.}}
\end{center}
When PGD is slow, memory provides polynomial savings; when fast, those savings come at exponential memory cost.
As the next section formalizes, this exchange rate is set by two structural inputs: the optimizer's convergence rate and the solution-map sensitivity.

\subsection{Meta proof framework}\label{meta-framework}
The proofs of Theorems~\ref{thm1} and \ref{thm2} share a common skeleton that we distill into a \emph{meta-theorem} (Appendix~\ref{meta proof}): combining two structural inputs --- the online optimizer's convergence rate and the solution-map sensitivity --- with covering and packing arguments on $\Theta$, it yields matching upper and lower bounds on $M(K,\varepsilon)$. New settings can be analyzed by verifying these two inputs; Theorems~\ref{thm1} and \ref{thm2} are themselves two such instantiations. The proof sketches below trace how the two inputs combine in each direction.

\myparagraph{$\mathcal{O}(\cdot)$ proof.}
The \emph{convergence rate} of PGD bounds the post-$K$-step error \emph{from above} by a decreasing function of $K$ and the initial optimality gap $f(\bx_0;\theta) - f^\star(\theta)$. The \emph{sensitivity} of the solution map (Lemma~\ref{lem3}) further bounds this initial gap \emph{from above} by $\|\theta - \pi_\mathcal{D}(\theta)\|$. Together, these two upper bounds reduce $\varepsilon$-accuracy to a covering condition on $\Theta_M$, yielding a sufficient memory $M(K,\varepsilon)$.

\myparagraph{$\Omega(\cdot)$ proof.}
The matching lower bound uses lower-bound counterparts of the same two inputs. We construct a hard function class $\underline{f}(\bx;\theta)$ on which (i) the \emph{convergence rate} is bounded \emph{from below}, so the post-$K$-step error cannot decay faster than a prescribed function of $K$ and the initial gap, and (ii) the \emph{sensitivity} is bounded \emph{from below}, so the initial gap is at least an increasing function of $\|\theta - \pi_\mathcal{D}(\theta)\|$. A packing argument on $\Theta$ then shows that any $M$ below $M(K,\varepsilon)$ leaves some $\theta$ for which $K$ PGD steps cannot reach $\varepsilon$-accuracy.

\newif\ifincludenn
\includennfalse 

\section{Numerical experiments}\label{experiment}

\ifincludenn

We validate the framework in two complementary settings: a controlled testbed on Tikhonov-regularized regression, where every structural driver admits closed-form expressions, and an applied setting on convex neural network training, where the framework extends naturally beyond synthetic problems.

\myparagraph{Controlled validation on Tikhonov regularization.}
We consider the parametric Tikhonov-regularized regression problem
\begin{equation}\label{Tihaov}
\min_{\bx\in\mathcal{X}} \;
\frac{1}{\beta}\|A\bx-\theta\|_2^\beta+\frac{\lambda}{\beta}\|\bx\|_2^\beta,
\end{equation}
on the Euclidean balls $\Theta = \mathbb{B}_2(0, R_\Theta) \subset \mathbb{R}^{6}$ and $\mathcal{X} = \mathbb{B}_2(0, R_\mathcal{X}) \subset \mathbb{R}^{12}$, with $R_\Theta = 1$, $R_\mathcal{X} = 1.5$, $\lambda = 0.3$, and $\beta \in \{2, 4, 6\}$. Setting $\beta = 2$ recovers standard ridge regression --- the strongly convex case of Theorem~\ref{thm1} --- while $\beta > 2$ yields a $\beta$-growth instance covered by Theorem~\ref{thm2}. The design matrix $A \in \mathbb{R}^{6 \times 12}$ is constructed via $A = U\Sigma V^\top$ with $U = I_6$, fixed singular values $(1, 0.9, 0.8, 0.7, 0.6, 0.5)$, and $V \in \mathbb{R}^{12 \times 12}$ a random orthogonal matrix. This ensures all structural drivers ($L_\Theta$, $L_{f,2}$, $\mu$, $\Phi$) admit closed-form expressions in $\Sigma$, $\lambda$, $R_\Theta$, $R_\mathcal{X}$ --- for instance, in the ridge case, $L_{f,2} = \sigma_{\max}^2(A) + \lambda = 1.3$ and $\mu = \lambda = 0.3$. For each memory size $M$ we draw $N = 100$ test parameters $\theta$, run warm-started PGD, and record the smallest $K$ at which the iterate reaches $\varepsilon$-accuracy; we report the mean and standard deviation of this stopping time.

\myparagraph{Application to convex neural networks.}
As shown in \cite{stanford_convex_nn}, the training of two-layer neural networks admits an equivalent convex reformulation under appropriate regularization. This places neural network training within the parametric convex framework studied here, with the data or task specification playing the role of the parameter $\theta$. Our memory--computation bounds therefore yield a direct theoretical handle on how much pre-computed information is needed to accelerate training across related tasks. Detailed settings are in Appendix~\ref{fullexperiment}.

\begin{figure*}[t]
    \centering
    \subfigure[Tikhonov, $\beta = 2$ (ridge)]{
        \includegraphics[width=0.325\linewidth]{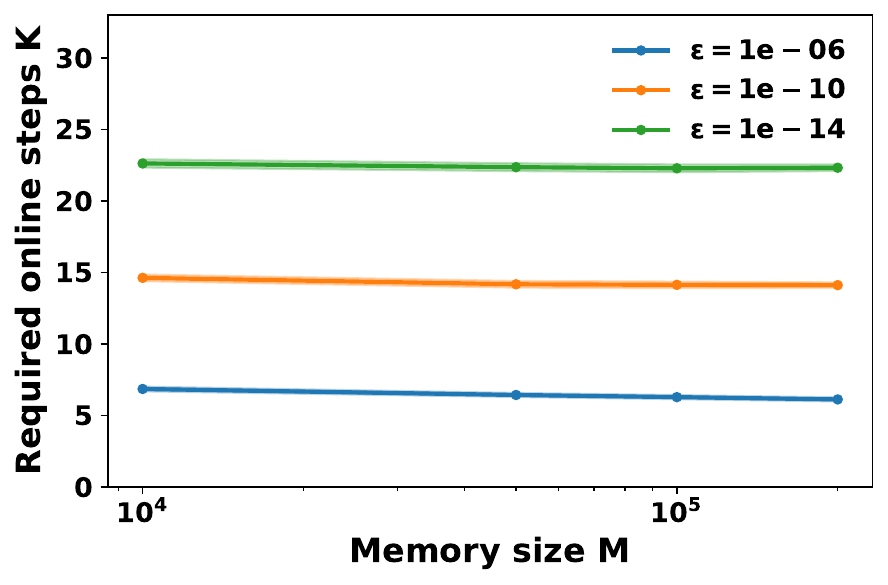}
        \label{fig:tik_b2}}
        \hspace{-3.5mm}
    \subfigure[Tikhonov, $\beta = 4$]{
        \includegraphics[width=0.325\linewidth]{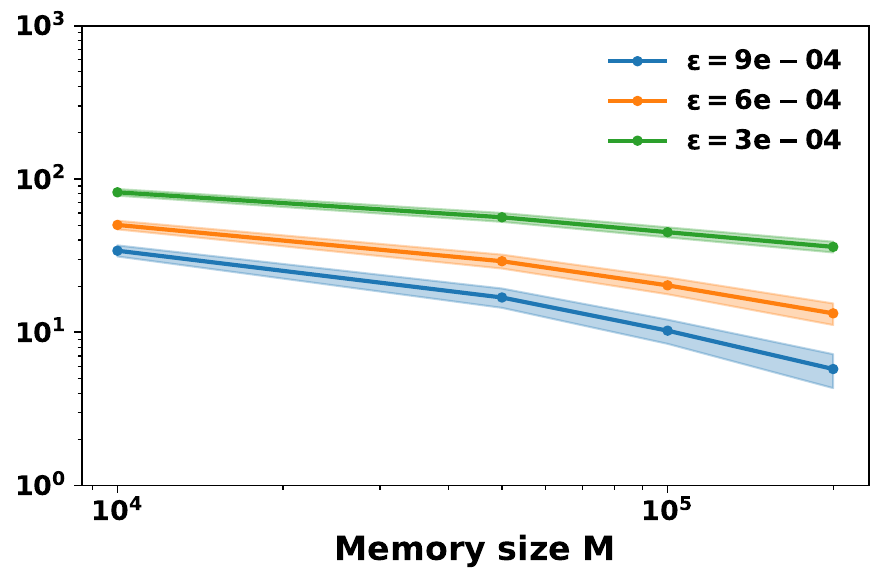}
        \label{fig:tik_b4}}
        \hspace{-3.5mm}
    \subfigure[Tikhonov, $\beta = 6$]{
        \includegraphics[width=0.325\linewidth]{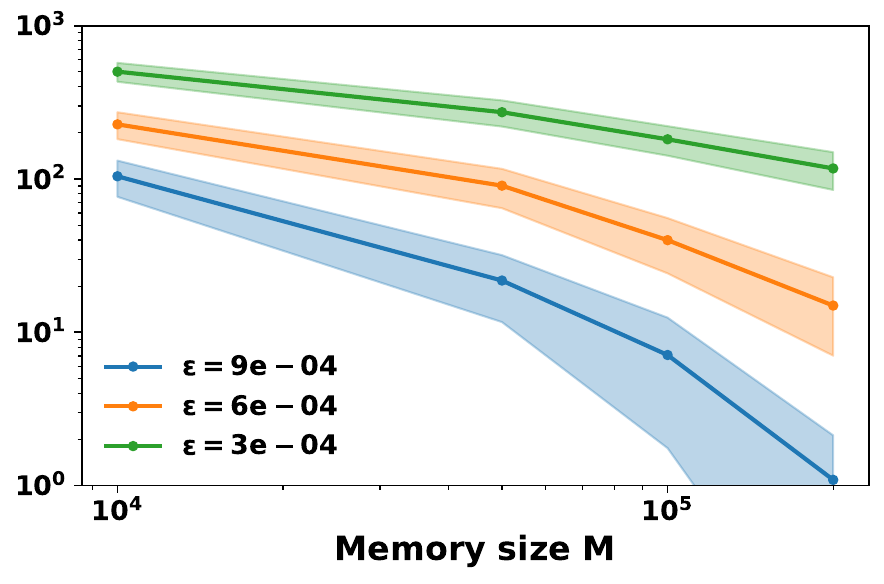}
        \label{fig:tik_b6}}\\
    \subfigure[Convex NN, $\beta = 2$]{
        \includegraphics[width=0.325\linewidth]{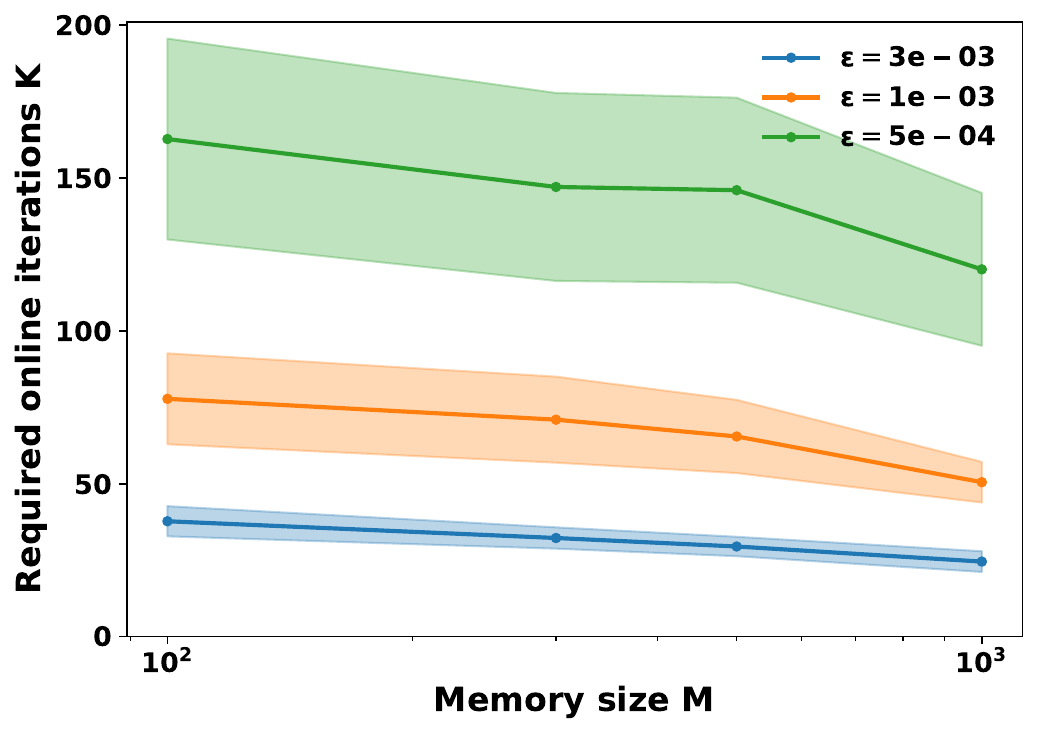}
        \label{fig:cnn_b2}}
        \hspace{-3.5mm}
    \subfigure[Convex NN, $\beta = 4$]{
        \includegraphics[width=0.325\linewidth]{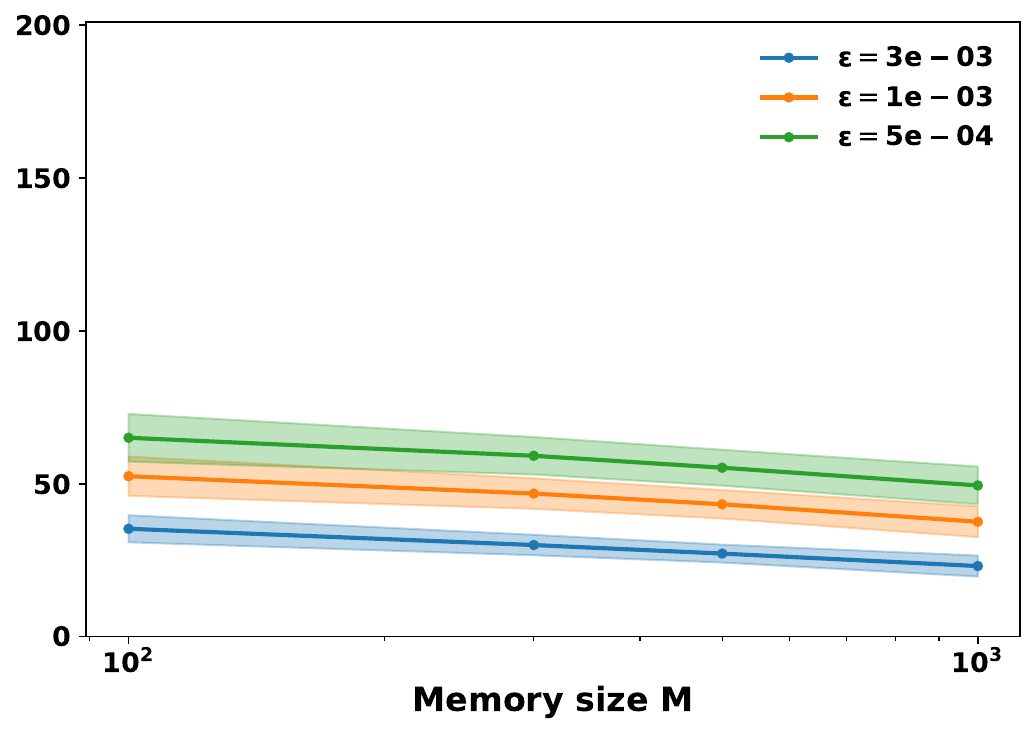}
        \label{fig:cnn_b4}}
        \hspace{-3.5mm}
    \subfigure[Convex NN, $\beta = 6$]{
        \includegraphics[width=0.325\linewidth]{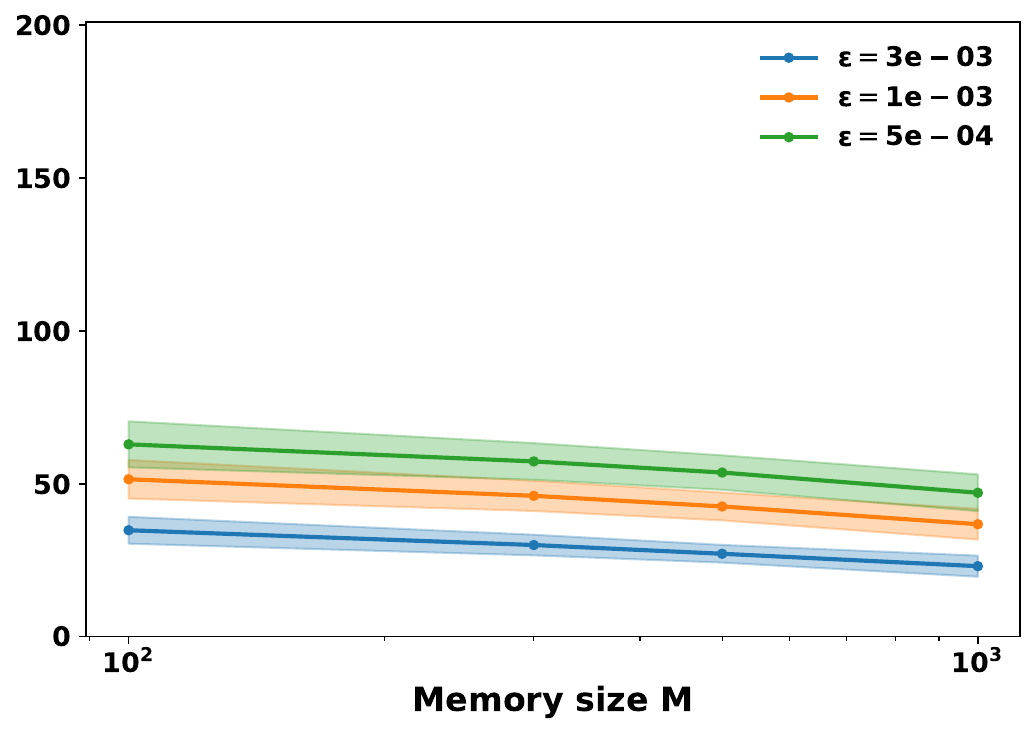}
        \label{fig:cnn_b6}}
    \caption{Average $K$ to reach accuracy $\varepsilon$ as a function of memory size $M$. \emph{Top} (a--c): Tikhonov family \eqref{Tihaov} with $\beta\in\{2,4,6\}$; (a) is ridge regression. \emph{Bottom} (d--f): convex two-layer NN training with $\beta\in\{2,4,6\}$. Shaded bands show $\pm 1$ standard deviation over $N=100$ (top) and $N=50$ (bottom) samples of $\theta$.}
    \label{fig:experiments}
\end{figure*}

Figure~\ref{fig:experiments} confirms the predicted tradeoff in both settings. In the strongly convex cases (panels a, d), the required $K$ depends only weakly on $M$: PGD's already-fast linear convergence leaves little room for amortization, consistent with the exponential memory cost of Theorem~\ref{thm1}. In the $\beta > 2$ panels (b, c, e, f), the picture is markedly different: increasing $M$ substantially reduces the required $K$, and this gain strengthens as $\beta$ grows --- matching the polynomial scaling of Theorem~\ref{thm2} and instantiating the substitution principle of Section~\ref{accel}. The pattern persists across the synthetic Tikhonov family (top) and the convex NN training task (bottom), supporting the meta-theorem's prediction that the substitution rate is determined by the online convergence regime rather than the specific problem class.

\else

We validate the framework on a controlled testbed: parametric Tikhonov-regularized regression, where every structural driver admits closed-form expressions, allowing a direct empirical check of the predicted memory--computation tradeoff in both the strongly convex (Theorem~\ref{thm1}) and $\beta$-growth (Theorem~\ref{thm2}) regimes.

\myparagraph{Setup.}
We consider the parametric Tikhonov-regularized regression problem
\begin{equation}\label{Tihaov}
\min_{\bx\in\mathcal{X}} \;
\frac{1}{\beta}\|A\bx-\theta\|_2^\beta+\frac{\lambda}{\beta}\|\bx\|_2^\beta,
\end{equation}
on the Euclidean balls $\Theta = \mathbb{B}_2(0, R_\Theta) \subset \mathbb{R}^{6}$ and $\mathcal{X} = \mathbb{B}_2(0, R_\mathcal{X}) \subset \mathbb{R}^{12}$, with $R_\Theta = 1$, $R_\mathcal{X} = 1.5$, $\lambda = 0.3$, and $\beta \in \{2, 4, 6\}$. Setting $\beta = 2$ recovers standard ridge regression --- the strongly convex case of Theorem~\ref{thm1} --- while $\beta > 2$ yields a $\beta$-growth instance covered by Theorem~\ref{thm2}. The design matrix $A \in \mathbb{R}^{6 \times 12}$ is constructed via $A = U\Sigma V^\top$ with $U = I_6$, fixed singular values $(1, 0.9, 0.8, 0.7, 0.6, 0.5)$, and $V \in \mathbb{R}^{12 \times 12}$ a random orthogonal matrix; this fixes the spectrum (and hence all structural drivers) while randomizing the geometry of the problem. The construction yields closed-form expressions for $L_\Theta$, $L_{f,2}$, $\mu$, $\Phi$ in $\Sigma$, $\lambda$, $R_\Theta$, $R_\mathcal{X}$ --- for instance, in the ridge case, $L_{f,2} = \sigma_{\max}^2(A) + \lambda = 1.3$ and $\mu = \lambda = 0.3$. For each memory size $M$ we draw $N = 100$ test parameters $\theta$, run warm-started PGD, and record the smallest $K$ at which the iterate reaches $\varepsilon$-accuracy; we report the mean and standard deviation of this stopping time.

\begin{figure*}[t]
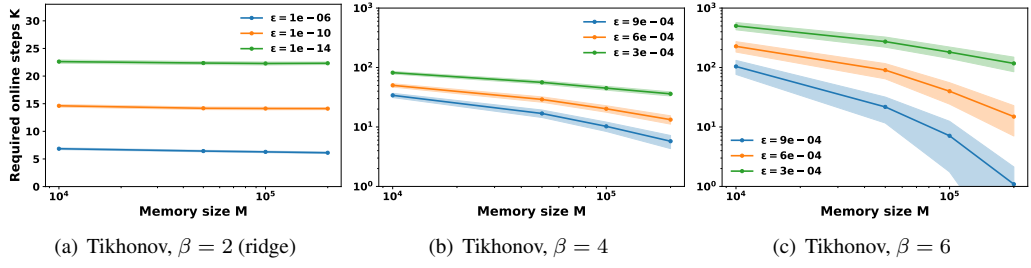

    \centering
    \subfigure[Tikhonov, $\beta = 2$ (ridge)]{
        \includegraphics[width=0.325\linewidth]{figs/required_K_vs_M_ridge_1NN.pdf}
        \label{fig:tik_b2}}
        \hspace{-3.5mm}
    \subfigure[Tikhonov, $\beta = 4$]{
        \includegraphics[width=0.325\linewidth]{figs/required_K_vs_M_beta4_loglog.pdf}
        \label{fig:tik_b4}}
        \hspace{-3.5mm}
    \subfigure[Tikhonov, $\beta = 6$]{
        \includegraphics[width=0.325\linewidth]{figs/required_K_vs_M_beta6_loglog.pdf}
        \label{fig:tik_b6}}
    \caption{Average $K$ to reach accuracy $\varepsilon$ as a function of memory size $M$ on the Tikhonov family \eqref{Tihaov}; (a) is ridge regression. Shaded bands show $\pm 1$ standard deviation over $N = 100$ samples of $\theta$.}
    \label{fig:tikhonov_experiment}
\end{figure*}

\myparagraph{Results.}
Figure~\ref{fig:tikhonov_experiment} confirms the predicted tradeoff. In the ridge case (a), the required $K$ depends only weakly on the memory size $M$: PGD's linear convergence at rate $\rho = 1 - L_{f,2}^{-1}\mu \approx 0.77$ already drives the error down geometrically, leaving little room for amortization. This is the qualitative signature of Theorem~\ref{thm1}'s exponential memory cost --- a meaningful reduction in $K$ would require memory growing exponentially in $K$, far beyond the range tested here. The $\beta > 2$ panels (b, c) tell a markedly different story: increasing $M$ substantially reduces the required $K$, and this gain strengthens as $\beta$ grows, consistent with the polynomial scaling of Theorem~\ref{thm2}. Together, the three panels instantiate the substitution principle of Section~\ref{accel}: when the online optimizer is already fast (ridge), additional memory provides only marginal benefit; when it is slow (high $\beta$), memory becomes a much cheaper substitute for online iterations.



\fi

\section{Conclusion, Limitations, and Future Work}\label{conclusions}

This work initiates a complexity-theoretic study of memory in semi-amortized parametric optimization, asking how much offline memory is needed to accelerate an online optimizer beyond its native rate. Our bounds for the $\mu$-strongly convex and $\beta$-growth regimes reveal that memory acts as a substitute for online computation, but at an exchange rate set by the optimizer's convergence rate and the solution-map sensitivity --- amortization helps most where online computation is expensive. We view this paper as a foundation for a broader research program; several natural directions lie ahead.

\myparagraph{Limitations and future work.}
The convexity assumption could be relaxed to non-convex regimes admitting comparable convergence rates, since the meta-theorem treats the rate as a black-box input. The exact-oracle assumption could be replaced by approximate or noisy offline solutions. Parameter-dependent feasible sets, ruled out here, are central to applications such as model predictive control. The dimensional scaling $(\cdot)^{d_\Theta}$ could be tempered by exploiting low-dimensional structure on $\Theta$. Finally, learned warm-start predictors --- neural networks, meta-learning initializations, kernel methods --- raise the question of how predictor expressiveness trades off against memory complexity.

\bibliography{references}
\bibliographystyle{plain}
\appendix
\section{Oracle assumption gap between uniform grid sampling and uniform random sampling}
\label{appendixA}
All the complexity arguments presented in Theorems~\ref{thm1} and \ref{thm2} are based on constructing a cover of $\Theta$.
In Algorithm~\ref{alg:data_collector}, sampling points are drawn from
the grid coordinates of a pre-constructed uniform mesh-grid
$\Theta_{M,\mathrm{Uni}}$, which contains $M$ points in total. However, in the numerical experiment (Section \ref{experiment}), we apply a random version of uniform sampling described as below. This algorithm is easier to implement by requiring weaker oracle than Algorithm~\ref{alg:data_collector}.
\begin{algorithm}[H]
\caption{Offline $\Theta$-uniform random sampler}
\label{alg:data_collector2}
\begin{algorithmic}
\REQUIRE Sampling budget $M_{\mathrm{STOC}}$.
\STATE \textbf{Initialize:} $\mathcal{D}=\emptyset$.
\FOR{$m=1,\ldots,M_{\mathrm{STOC}}$}
    \STATE Sample $\theta_m \sim \mathrm{Uniform}(\Theta)$.
    \STATE Solve $\bx_m^\star\in\arg\min_{\bx\in\mathcal{X}}f(\bx;\theta_m)$.
    \STATE Add the data tuple $(\theta_m,\bx_m^\star)$ to dataset $\mathcal{D}$.
\ENDFOR
\STATE \textbf{Output:} Nonparametric predictor $\pi_{\mathcal{D}}(\cdot)\gets\arg\min_{\theta_m\in\Theta_{M_{\mathrm{STOC}}}} \left\{\|\cdot-\theta_m\|\right\}$.
\end{algorithmic}
\end{algorithm}
While this randomized strategy relieves the requirement of constructing $\Theta_{M,\mathrm{Uni}}$, it incurs an additional logarithmic factor in sample complexity. In particular, to form an $r$-cover of $\Theta$ with probability at least $1-\delta$, it requires on the order of
\[
M_{\mathrm{STOC}}\leq\mathcal{O}\!\left(\mathbb{M}_{\mathrm{cover}}(\Theta;r)\,\log\!\left(\mathbb{M}_{\mathrm{cover}}(\Theta;r)\right)\,\log\!\left(\tfrac{1}{\delta}\right)\right)
\]
samples, where $\mathbb{M}_{\mathrm{cover}}(\Theta;r)$ denotes the covering number of $\Theta$ at radius $r$. 
The uniform grid sampling mechanism (Algorithm~\ref{alg:data_collector}) appears to enjoy superior sample complexity guarantees (without $\log\left(\frac{1}\delta\right)$). In particular, it can construct an $r$-cover of the parameter domain $\Theta$ using on the order of $\mathcal{O}(\mathbb{M}_{\mathrm{cover}}(\Theta;r))$ samples, which is information-theoretically minimax optimal. However, constructing such a uniform grid $\Theta_{M,\mathrm{Uni}}$ implicitly requires access to a substantially stronger oracle. While Algorithm~\ref{alg:data_collector2} only assumes a sampling oracle capable of generating i.i.d.\ samples $\theta_m \in \Theta$, Algorithm~\ref{alg:data_collector} additionally requires access to geometric information about the parameter domain in order to construct an outer grid $\Theta_{\mathrm{Uni}}$ that contains $\Theta$ (see Figure \ref{fig:grid_sample}). Then, the uniform grid can be constructed by $\Theta_{M,\mathrm{Uni}}=\Theta\cap\Theta_{\mathrm{Uni}}$. These considerations make Algorithm~\ref{alg:data_collector} demanding in terms of extra assumptions.
\begin{figure}[htb]
    \centering
    \includegraphics[width=1\linewidth]{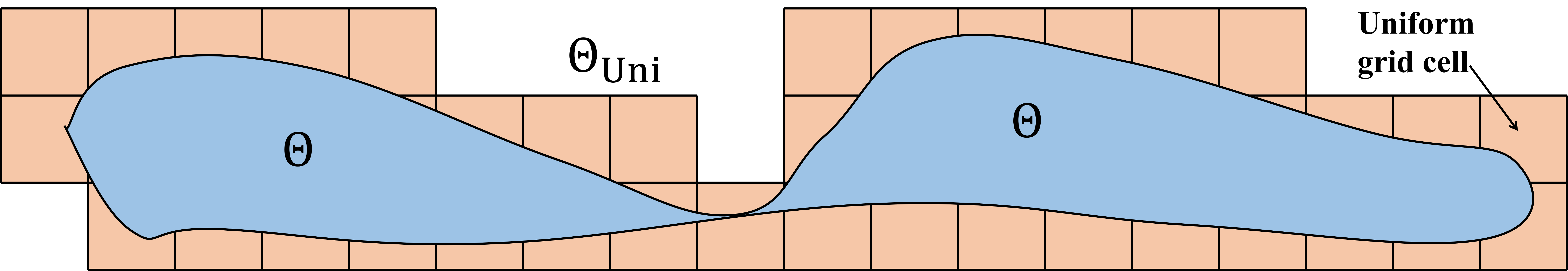}
    \caption{The uniform grid cover $\Theta_{M,\mathrm{Uni}}$ is constructed over an irregular set $\Theta$. }
    \label{fig:grid_sample}
\end{figure}

\myparagraph{Geometric regularity} In order to use Algorithm \ref{alg:data_collector2} for random sampling over $\Theta$, we need the following assumption to avoid ill-condition sets. 
\begin{assumption}[Geometric regularity (Definition 2.2 in \cite{alvarado2021geometry})]\label{a10}
We assume that for $\Theta \subset \mathbb{R}^{d_\Theta}$, its complement $\Theta^c$ has positive reach, i.e., $\mathrm{reach}(\Theta^c) \ge r$ for some $r>0$. 
\end{assumption}
Positive reach excludes pathological geometric features such as cusps and arbitrarily thin spike-like structures. In particular, it guarantees that the distance function to the boundary is well behaved and that nearest-point projections are uniquely defined in a neighborhood of the domain. As a result, the domain does not contain regions of vanishing thickness, which prevents degeneracies in covering and partition arguments (e.g., cells with arbitrarily small volume) and enables uniform control of cell diameter and volume.

Under the positive reach condition, Theorem~3.7 in \cite{alvarado2021geometry} ensures that the probabilistic covering captures the entire set with high probability, without leaving uncovered regions of negligible measure.

\myparagraph{Lower bound for covering number}

First, from the pigeon principle, we have the following lemma to describe the least number we need to construct a $r$-cover on $\Theta$.
\begin{lem}\label{lem1-}
Let $v_{d_\Theta}$ denote the Lebesgue measure of the unit ball in $\mathbb{R}^{d_\Theta}$ under the $\|\cdot\|$ norm. 
If $\{\theta_m\}_{m=1}^M \subset \Theta$ forms an $r$-cover of $\Theta$ for a given $r$, then
\[
M \ge \frac{|\Theta|}{v_{d_\Theta} r^{d_\Theta}}.
\]
Equivalently, for any given $M$ that forms an $r$-cover on $\Theta$, $r$ must satisfy
\[
r 
\ge \left(\frac{|\Theta|}{v_{d_\Theta}}\right)^{1/d_\Theta} M^{-1/d_\Theta}.
\]
\end{lem}

\begin{proof}
 To form an $r$-cover, we necessarily require
\[
\Theta \subset \bigcup_{m=1}^M \mathbb{B}(\theta_m,r),
\]
where $\mathbb{B}(\theta_m,r)$ are balls centered at sampling points $\Theta_M=\{\theta_m\}_{m=1}^M\subset\Theta$. This implies
\begin{align*}
|\Theta|
\le \left|\bigcup_{m=1}^M \mathbb{B}(\theta_m,r)\right|
\le \sum_{m=1}^M |\mathbb{B}(\theta_m,r)|
= M v_{d_\Theta} r^{d_\Theta}.
\end{align*}
By rearranging terms, only $M \ge \left(\frac{|\Theta|}{v_{d_\Theta}}\right)\frac{1}{r^{d_\Theta}}$ can construct a cover with sufficient volume to cover $\Theta$. On the other hand, for a given $M$,
\[
\sum_{m=1}^M \left|\mathbb{B}\left(\theta_m,\left(\frac{|\Theta|}{v_{d_\Theta}}\right)^{1/d_\Theta} M^{-1/d_\Theta}\right)\right|
\ge |\Theta|.
\]
Otherwise, the covering fails.
\end{proof}
However, we can't claim the relationship between $R_{\Theta}$ and $|\Theta|$ without the following condition. Inspired from Proposition 6 in \cite{castellano2025data}, this condition controls the geometry of the parameter space $\Theta$.
\begin{dfn}[$\gamma$-scaled ball containment]\label{dfn2}
There exist a $\theta\in\Theta$ and $\gamma\in (0,1]$ to let,
\[\mathbb{B}(\theta,\gamma R_{\Theta})\subset \Theta.\]
\end{dfn}
Under Definition \ref{dfn2}, we can rewrite Lemma \ref{lem1-} next for a more conservative but quantified version.
\begin{lem}
Suppose the $\gamma$-scale ball containment holds, denoting $v_{d_\Theta}$ as the Lebesgue measure of the unit ball in $\mathbb{R}^{d_\Theta}$ under the $\|\cdot\|$ norm. 
If $\{\theta_m\}_{m=1}^M \subset \Theta$ forms an $r$-cover of $\Theta$ for a given $r$, then
\[
M \ge \left(\frac{\gamma R_\Theta}{r}\right)^{d_\Theta}.
\]
Equivalently, for any given $M$ that forms an $r$-cover on $\Theta$, $r$ must satisfy
\[
r 
\ge \gamma R_\Theta M^{-1/d_\Theta}.
\]
\end{lem}
In the subsequent lower bound constructions, we choose $\Theta$ as a ball under the norm $\|\cdot\|_2$.
\section{Meta proof framework (full version)}\label{meta proof}
In this section, we introduce a meta-theoretic framework of proving Theorem \ref{thm1} \& \ref{thm2}.  

\myparagraph{Step 1: Algorithmic convergence (upper bound).}
Assume that the PGD generates a sequence $\{\bx_k(\theta)\}_{k=1}^K$ such that, for any given parameter $\theta \in \Theta$,
\begin{equation}
\min_{k\in[K]} f(\bx_k(\theta);\theta) - f^\star(\theta)
\le
\overline{\mathcal{E}}\!\left(K,\, f(\bx_0(\theta);\theta) - f^\star(\theta)\right),
\nonumber
\end{equation}
where $\overline{\mathcal{E}}\!\left(K,\, f(\bx_0(\theta);\theta) - f^\star(\theta)\right)$ is a nonnegative function that is nonincreasing in the number of online oracle cost $K$ and nondecreasing in the initial optimality gap $f(\bx_0(\theta);\theta) - f^\star(\theta)$. 
This function characterizes both the convergence rate of the algorithm and its sensitivity to the initial optimality of the warm-start $\bx_0(\theta)$.

\myparagraph{Step 2: Parametric regularity.}
To derive the upper bound, we first formalize the Lipschitz continuity of the initial optimality gap with respect to the parameter.
\begin{lem}[Lipschitz continuity of optimality gap]\label{lem3}
Under Assumption \ref{a1}-\ref{a4},
\[|f(\bx^\star(\pi_\mathcal{D}(\theta)),\theta)-f^{\star}(\theta)|
\;\le\; 2L_\Theta\Vert\pi_\mathcal{D}(\theta)-\theta\Vert,\,\theta\in\Theta.\]
\end{lem} 
The proof is shown in Appendix \ref{appendix0}.

For the purpose of establishing lower bounds, we construct a parameterized function class $\underline{f}(\bx;\theta)$ that simultaneously satisfies the following conditions.
\begin{itemize}
\item  The corresponding online refinement sequence $\{\bx_k(\theta)\}_{k=1}^K$ satisfies, \[\min_{k\in[K]}\underline{f}(\bx_k(\theta);\theta) - \underline{f}^\star(\theta)
\;\geq\;
\underline{\mathcal{E}}(K,\underline{f}(\bx_0(\theta);\theta) - \underline{f}^\star(\theta)),\]where $\underline{\mathcal{E}}\!\left(K, f(\bx_0;\theta) - f^\star(\theta)\right)$ is a non-negative, decreasing function of online oracle cost $K$ and an increasing function of the initial optimality $f(\bx_0(\theta);\theta) - f^\star(\theta)$.
\item There exists a non-negative, monotonic increasing modulus $\mathcal{R}_{\mathrm{lb}}(\cdot)$ such that \[\mathcal{R}_{\mathrm{lb}}\!\left(\|\pi_{\mathcal{D}}(\theta) - \theta\|_{2}\right) \leq \underline{f}(\bx^\star(\pi_{\mathcal{D}}(\theta));\theta) - \underline{f}^\star(\theta)=\underline{f}(\bx_0(\theta);\theta) - \underline{f}^\star(\theta)  ,\,\forall \theta\in \Theta.\]
\end{itemize}
 By norm equivalence in finite dimensions (Theorem~2.1 in \cite{jin2016introduction}), packing and covering bounds under $\|\cdot\|_2$ change only up to constants with respect to $\|\cdot\|$.

\myparagraph{Step 3: Proof framework (upper bound).}
The covering distance $\|\pi_{\mathcal{D}}(\theta)-\theta\|$ is required to ensure that the last inequality in the following chain holds:
\begin{align}\label{upperchain}
\min_{k\in[K]} f(\bx_k(\theta);\theta) - f^\star(\theta)
\le{}
\overline{\mathcal{E}}\!\left(K, f(\bx^\star(\pi_\mathcal{D}(\theta));\theta) - f^\star(\theta)\right)\nonumber\le
\overline{\mathcal{E}}\!\Bigl(K,
2L_\Theta\|\pi_{\mathcal{D}}(\theta)-\theta\|
\Bigr)
\le \varepsilon 
\end{align}
We denote by $M$ as the sampling budget of a uniform grid sampling (Algorithm \ref{alg:data_collector}) required to ensure that the minimum separating distance $\|\pi_{\mathcal{D}}(\theta)-\theta\|$ is sufficiently small. This quantity characterizes the sufficient number of uniformly spaced samples needed to form a $(K,\varepsilon)$-$\Theta$-net. Consequently, a memory $M \ge M(K,\varepsilon)$ guarantees that the prescribed error tolerance $\varepsilon$ can be achieved after $K$ steps of online refinement.

\myparagraph{Step 3: Proof framework (lower bound).} For the lower bound, we construct a parameterized function class $\underline{f}(\bx_k(\theta);\theta)$ satisfying the three conditions stated above. If there exists a parameter $\theta\in\Theta$
such that the covering distance $\|\pi_{\mathcal{D}}(\theta)-\theta\|_{2}$ is sufficiently large, then
{\fontsize{9.5pt}{12pt}\selectfont{\begin{align*}
\min_{k=1,\ldots,K}\underline{f}(\bx_k(\theta);\theta) - \underline{f}^\star(\theta)
=
\underline{\mathcal{E}}\!\Bigl(K,
\underline{f}(\bx^\star(\pi_{\mathcal{D}}(\theta));\theta)
- \underline{f}^\star(\theta)\Bigr)\ge
\underline{\mathcal{E}}\!\Bigl(K,
\mathcal{R}_{\mathrm{lb}}\!\left(\|\pi_{\mathcal{D}}(\theta) - \theta\|_{2}\right)
\Bigr)
> \varepsilon .
\end{align*}}}

Consequently, any sampling mechanism with budget $M$ is insufficient to form a $(K,\varepsilon)$-$\Theta$-net. In particular, for any cover with cardinality $M$, there exists a region of $\Theta$ that remains uncovered, such that
\[
\min_{k=1,\ldots,K} \underline{f}(\bx_k(\theta);\theta) - \underline{f}^\star(\theta) > \varepsilon .
\]
Therefore, it follows that $M < M(K,\varepsilon)$.
\section{Proof of Lemma \ref{lem3}}
\label{appendix0}
First, we construct the Lipschitz continuity for $f^\star(\theta)$ as follows.
\begin{lem}[Lipschitz continuity of $f^\star(\cdot)$]\label{lem4}
Under Assumption \ref{a1}-\ref{a4}, $f^\star(\theta)$ is Lipschitz continuous with constant $L_\Theta$ such that,
\[|f^\star(\theta_1)-f^{\star}(\theta_2)|
\;\le\; L_\Theta\Vert\theta_1-\theta_2\Vert,\,\theta_1,\,\theta_2\in\Theta.\]
\end{lem}
\begin{proof}
We could start from the definition of $f^\star(\theta)$ in \eqref{cpp},
\begin{align*}
|f^\star(\theta_1)-f^\star(\theta_2)|&=|\min_{\bx\in\mathcal{X}} f(\bx,\theta_1)-\min_{\bx\in\mathcal{X}}f(\bx,\theta_2)|\\
&\leq\max_{\bx\in\mathcal{X}}|f(\bx,\theta_1)-f(\bx,\theta_2)|\\
&\leq L_\Theta\|\theta_1-\theta_2\|.
\end{align*}
Hence, we finished the proof.
\end{proof}
Then, we could begin the proof. By applying Assumption \ref{a3} and Lemma \ref{lem4}, respectively, we have,
\begin{align*}
|f(\bx^\star(\pi_\mathcal{D}(\theta)),\theta)-f^{\star}(\theta)|
&\le |f(\bx^\star(\pi_\mathcal{D}(\theta)),\theta)-f^\star(\pi_\mathcal{D}(\theta))|+|f^\star(\pi_\mathcal{D}(\theta))-f^{\star}(\theta)|\\&\le 2L_\Theta\|\pi_\mathcal{D}(\theta)-\theta\|.
\end{align*}
\section{Proof of Theorem \ref{thm1}}
\label{appendix1}
\myparagraph{Complexity upper bound} First, for the gradient descent in strongly convex case (Assumption \ref{a6}), we have the following linear convergence result (see Section 9.3.1 in \cite{boyd2004convex}),
\[\min_{k\in[K]}f(\bx_k;\theta)-f^\star(\theta)\leq\rho^{K}(f(\bx_0;\theta)-f^\star(\theta)),\,\forall \bx_0\in\mathcal{X},\,\theta\in\Theta,\]
where $f(\bx_k;\theta)$ is the function value of the $k$-th iteration of projected gradient descent\footnote{To ease the burden of notation, we use $\bx_k$ rather than $\bx_k(\theta)$ for the sequence generated by PGD on $f(\bx;\theta)$.}. Then, for the online oracle cost $K$ and error tolerance $\varepsilon$, it is sufficient to let $f(\bx_0;\theta)-f^\star(\theta)\leq \varepsilon\rho^{-K}$ to meet the requirement of $(K,\varepsilon)$-$\Theta$-net. To bound $f(\bx_0;\theta)-f^\star(\theta)$, from the Lemma \ref{lem3}, we suppose $\bx_0=\bx_{m^\star}^\star,\,\theta_{m^{\star}}=\pi_{\mathcal{D}}(\theta)$, then,
\begin{align*}
f(\bx_0;\theta)-f^\star(\theta)\leq 2L_\Theta\Vert \pi_{\mathcal{D}}(\theta)-\theta\Vert,
\end{align*}
where $\pi_{\mathcal{D}}(\theta)$ is the $\theta$ chosen by the nearest-neighborhood rule. So we are sufficiently to let,
\begin{align*}
f(\bx_0;\theta)-f^\star(\theta)\leq 2L_\Theta\Vert \pi_{\mathcal{D}}(\theta)-\theta\Vert&\leq\varepsilon \rho^{-K}\\
\Vert \pi_{\mathcal{D}}(\theta)-\theta\Vert&\leq\frac{\varepsilon\rho^{-K}}{2L_\Theta},
\end{align*}
to construct a satisfactory $(K,\varepsilon)$-$\Theta$-net. If we have totally $M$-points and it is sampled by Algorithm \ref{alg:data_collector}, then the covering distance satisfies $\Vert \pi_{\mathcal{D}}(\theta)-\theta\Vert\lesssim R_{\Theta}{M}^{-1/d_\Theta}$. As a consequence, we are sufficient to ensure,
\[M^{-1/d_\Theta}\lesssim\frac{\varepsilon\rho^{-K}}{2R_{\Theta}L_\Theta}.\]
By rearranging terms, we finish the proof of the upper bound complexity $\mathcal{O}(\cdot)$.
\begin{figure}[htb]
    \centering
    \includegraphics[width=0.35\linewidth]{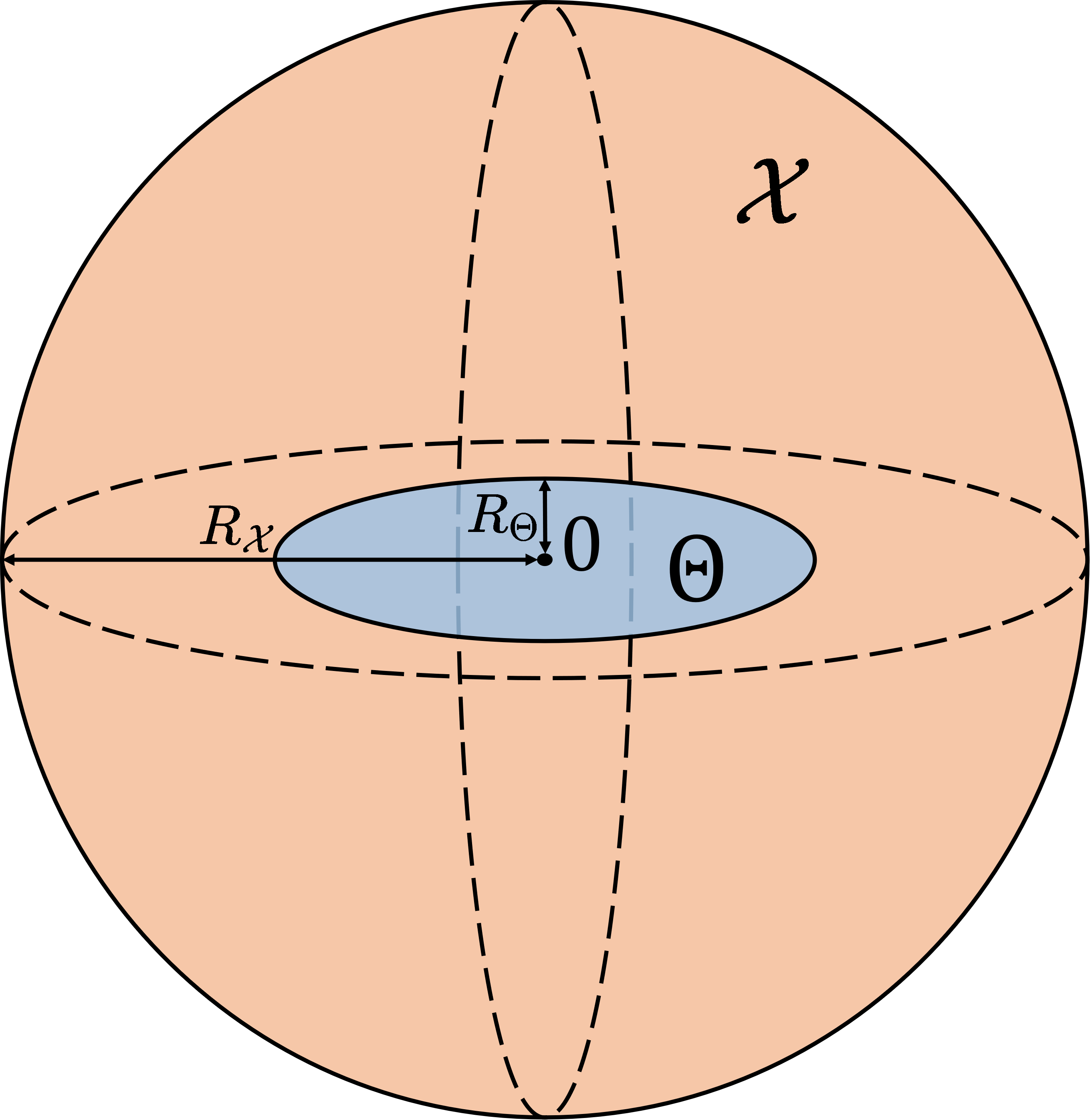}
    \caption{The examples used to establish the $\Omega(\cdot)$ lower bounds in Theorems~\ref{thm1}, \ref{thm2} are constructed based on this figure. 
Both $\mathcal{X}$ ({\color{orange}{orange}}) and $\Theta$ ({\color{blue}{blue}}) are taken to be balls centered at the origin, with dimensions $d_\Theta < d_{\mathcal{X}}$ and radii $R_{\Theta} \le R_{\mathcal{X}}$, respectively. 
Under this construction, we can design a class of functions $f(\bx;\theta)$ such that the corresponding optimal solution $\bx^\star(\theta)$ lies in $\Theta$, which ensures that the PGD trajectory induced by the warm-up solution remains confined to $\Theta$ and can be directly compared with the optimal trajectory of PGD.
}
    \label{fig:lb_example}
\end{figure}

\myparagraph{Complexity lower bound} For the lower bound $\Omega(\cdot)$, we consider the parametric function class
\[
\underline{f}(\bx;\theta)
=
\left\Vert
\bx
\right\Vert_Q^2-\frac{L_\Theta}{R_\mathcal{X}}\left\langle\bx,\begin{bmatrix}
\theta \\ \mathbf{0}
\end{bmatrix}\right\rangle,
\quad
\bx \in \mathbb{B}_2(0,R_{\mathcal{X}}) \subset \mathbb{R}^{d_{\mathcal{X}}},
\quad
\theta \in \mathbb{B}_2(0,R_{\Theta}) \subset \mathbb{R}^{d_\Theta},
\]
where $d_{\mathcal{X}} \geq d_\Theta$ and $R_{\mathcal{X}}\geq \sqrt{\frac{L_\Theta R_\Theta}{\mu }}$.\footnote{
$\mathbb{B}_Q(\bx,r)$ denotes the ball centered at $\bx$ with radius $r$ under the norm $\|\cdot\|_Q$. The same notation applies to $\mathbb{B}_Q(\theta,r)$ in the parameter space $\Theta$.
}
Here, $\|\cdot\|_Q$ denotes the Lyapunov norm defined as $\|\bx\|_Q = \sqrt{\bx^\top Q \bx}$ with $Q = \mathrm{diag}\left( \frac{\mu}{2} I_{d_\Theta},\frac{L_{f,2}}{2} I_{d_{\mathcal{X}}-d_\Theta}\right)$. Next, we check whether $\underline{f}(\bx;\theta)$ satisfies our assumptions.
\begin{itemize}
\item $L_{f,2}$-smoothness on $\bx\in\mathcal{X}$:
\[\nabla_\bx^2 \underline{f}(\bx;\theta)\preceq L_{f,2} I_{d_\mathcal{X}}. \]
\item $\mu$-strongly convex:
\[\nabla_\bx^2\underline{f}(\bx;\theta)\succeq \mu I_{d_\mathcal{X}}.\]
\item $L_\Theta$-Lipschitz on $\theta\in\Theta$:
\[\|\nabla_\theta \underline{f}(\bx;\theta)\|_2=\left\|\frac{L_\Theta}{R_\mathcal{X}}\bx\right\|_2\le L_\Theta.\]
\end{itemize}
Consequently, $\underline{f}(\bx;\theta)$ is $\mu$-strongly convex and $L_{f,2}$-smooth with respect to $\bx$. And $L_\Theta$-Lipschitz corresponding to $\theta$. The construction idea is shown in Figure \ref{fig:lb_example}. Then, we consider the warm-up gap between $\underline{f}(\bx_0;\theta)$ and $\underline{f}^\star(\theta)$ as follows. Suppose $\bx_0=\bx^\star(\pi(\theta))$ for arbitrary predictor $\pi:\Theta\to\Theta_M$, given the analytical solution $\bx^\star(\theta) = \frac{L_\Theta}{\mu R_\mathcal{X}} \begin{bmatrix} \theta \\ 0 \end{bmatrix}$. Given $R_{\mathcal{X}}\geq \sqrt{\frac{L_\Theta R_\Theta}{\mu }}$, then $\bx^\star(\theta)\in \mathbb{B}_2(0,R_{\mathcal{X}})$. So we have,
\begin{align*}
\underline{f}(\bx_0;\theta)-\underline{f}^\star(\theta)&=\frac{L_\Theta^2}{\mu^2R_\mathcal{X}^2}\| \pi(\theta)\|_Q^2-\frac{L_\Theta^2}{\mu R_\mathcal{X}^2}\left\langle\pi(\theta),\,\theta\right\rangle+\frac{L_\Theta^2}{2\mu R_\mathcal{X}^2}\|\theta\|_2^2\\&=\frac{L_\Theta^2}{2\mu R_\mathcal{X}^2}\| \pi(\theta)\|_2^2-\frac{L_\Theta^2}{\mu R_\mathcal{X}^2}\left\langle\pi(\theta),\,\theta\right\rangle+\frac{L_\Theta^2}{2\mu R_\mathcal{X}^2}\|\theta\|_2^2\\&=\frac{L_\Theta^2}{2\mu R_\mathcal{X}^2}\|\pi(\theta)-\theta\|_2^2.
\end{align*}
Since $\pi^\star$ selects $\theta_m$ to minimize 
$\underline{f}(\bx_0;\theta)-\underline{f}^\star(\theta)$, in our example, it is sufficient to choose $\pi^\star=\pi_{\mathcal{D}}$ under the Euclidean norm. 

Moreover, since $\underline{f}(\bx;\theta)$ is quadratic and the parameter-induced variation is restricted to the slowest eigenspace for all $\theta \in \mathbb{B}_2(0,R_\Theta)$, the linear convergence rate of gradient descent is tight  with rate $\rho^2$ (see Section 9.3.2 in \cite{boyd2004convex}) because,
\begin{align*}
\bx_{k+1}
&=
\bx_k-\frac{1}{L_{f,2}}\nabla f(\bx_k)\\
\bx_{k+1}
&=
\bx_k-\frac{\mu}{L_{f,2}}\left(\bx_k-\frac{L_\Theta}{\mu R_\mathcal{X}} \begin{bmatrix} \theta \\ 0 \end{bmatrix}\right)\\
\bx_{k+1}-\frac{L_\Theta}{\mu R_\mathcal{X}} \begin{bmatrix} \theta \\ 0 \end{bmatrix}&=
\left(1-\frac{\mu}{L_{f,2}}\right)\left(\bx_k-\frac{L_\Theta}{\mu R_X} \begin{bmatrix} \theta \\ 0 \end{bmatrix}\right)\\
\left\|\bx_{k+1}-\frac{L_\Theta}{\mu R_\mathcal{X}} \begin{bmatrix} \theta \\ 0 \end{bmatrix}\right\|_2&=
\left|1-\frac{\mu}{L_{f,2}}\right|\left\|\bx_k-\frac{L_\Theta}{\mu R_\mathcal{X}} \begin{bmatrix} \theta \\ 0 \end{bmatrix}\right\|_2.
\end{align*}
because $\bx_k$ always with in the slow subspace of convergence, we have,
\[\underline{f}(\bx_k;\theta)-\underline{f}^\star(\theta)=\frac{\mu}2\left\|\bx_k-\frac{L_\Theta}{\mu R_\mathcal{X}} \begin{bmatrix} \theta \\ 0 \end{bmatrix}\right\|_2^2=\frac{L_\Theta^2\rho^{2K}}{2\mu R_\mathcal{X}^2}\|\pi_\mathcal{D}(\theta)-\theta\|_2^2.\]
Then, if $M$ points are sampled in $\Theta$, there exists $\theta\in\Theta$ that,
\[\left\Vert \pi_{\mathcal{D}}(\theta)-\theta \\ \right\Vert_2\gtrsim R_\Theta M^{-1/d_\Theta}.\]
Thus, we have,
\[\min_{k\in[K]}\underline{f}(\bx_k;\theta)-\underline{f}^\star(\theta)=\frac{L_\Theta^2 \rho^{2K}}{2\mu R_\mathcal{X}^2}\left\Vert \pi_{\mathcal{D}}(\theta)-\theta \\ \right\Vert_2^2\gtrsim \frac{L_\Theta^2 R_\Theta^2 M^{-2/d_\Theta} \rho^{2K}}{2\mu R_\mathcal{X}^2}>\varepsilon.\]
So, $M\lesssim\left(\frac{L_\Theta^2 R_{\Theta}^2}{\mu R_{\mathcal{X}}^2\rho^{-2K}\varepsilon}\right)^{d_\Theta/2}$ is clearly unacceptable in the requirement of $(K,\varepsilon)$-$\Theta$-net. So we have the memory complexity $M(K,\varepsilon)$ is lower bounded by,
\[\Omega\left(\left(\frac{L_\Theta R_\Theta}{R_\mathcal{X}}\right)^{d_\Theta}\left(\frac{1}{\mu\rho^{-2K}\varepsilon}\right)^{d_\Theta/2}\right)=\Omega\left(\left({L_\Theta R_\Theta}\right)^{d_\Theta}\left(\frac{1}{\rho^{-2K}\varepsilon}\right)^{d_\Theta/2}\right),\] where $\rho=1-\mu/L_{f,2}$.
\begin{rem}[Norm change in the lower bound proof]
In this paper, all lower bounds of memory complexities are derived by regularizing the distance $\|\pi^\star(\theta)-\theta\|_2$ under the Euclidean norm. For a general norm $\|\cdot\|$, the corresponding covering numbers remain of the same order due to norm equivalence in finite-dimensional spaces (see Theorem~2.1 in \cite{jin2016introduction}).
\end{rem}
\section{Results of general convex case}\label{appendix2}
 The qualitative memory–computation tradeoff for the general convex smooth case is characterized in the following Theorem \ref{thm3}.
\begin{thm}\label{thm3}
Under Assumptions~\ref{a1}--\ref{a4}, given the online oracle cost $K$, the memory
complexity sufficient to guarantee error tolerance $\varepsilon$ satisfies
\[
M(K,\varepsilon)
\le
\mathcal O\!\left(
\left(
\frac{R_\Theta L_\Theta}{K\varepsilon^2}
\right)^{d_\Theta}
\right).
\]
In addition, when restricted to the subclass of $\beta$-growth convex problems, the memory complexity in the general convex setting must satisfy all lower bounds established for $\beta>2$, that is,
\[
M(K,\varepsilon)
\ge
\Omega\!\left((\min\{L_{f,2},L_\Theta\}R_\Theta^2)^{\frac{d_\Theta}{\beta}}
\left[
\left(
\frac{1}{\varepsilon}
\right)^{\frac{\beta-2}{\beta}}
-C_2K
\right]_+^{\frac{d_\Theta}{\beta-2}}
\right),\,\forall \beta>2,
\]
\end{thm}
\noindent
where
$C_2 = \frac{c_{\mathrm{bound}} \, \beta (\beta - 2)\min\{L_{f,2}, L_\Theta\}}{(\min\{L_{\Theta},L_{f,2}\}R_{\Theta}^2)^{\frac{\beta-2}{\beta}}L_{f,2}}$,
with
$c_{\mathrm{bound}} =
\left(1-\frac{L_{f,\Theta}}{\beta-1}\left(0.865R_\Theta\right)^{\beta-2}\right)^{-2}$
when $\beta \in (2,3)$, and
$
c_{\mathrm{bound}} = 2$
otherwise.
\begin{proof}
\myparagraph{Complexity upper bound} We could prove this theorem from a well-known inequality (Theorem 2.3 in \cite{beck2010gradient}),
\begin{align}
\min_{k\in[K]}\Vert \bx_{k}-\bx_{k-1}\Vert_2&\leq \sqrt{\frac{2(f(\bx_0;\theta)-f^\star(\theta))}{L_{f,2}K},
}\label{gua1}
\end{align}
 where \eqref{gua1} is proved by summing the descent lemma of PGD (Lemma 2.6 in \cite{beck2010gradient}), 
 \[f(\bx_{k};\theta)\leq f(\bx_{k-1};\theta)-\frac{L_{f,2}}{2}\|\bx_{k}-\bx_{k-1}\|_2^2,\,K\geq k\geq1.\]
 from $1$ to $K$. The formula above holds for smooth $f(\bx;\theta)$ respect to $\bx$ (Assumption \ref{a3}). First, for any given initialization $\bx_0\in\mathcal{X}$, the PGD descent sequence $\bx_{k-1}$ should satisfy,
\begin{equation}
\begin{aligned}
f^\star(\theta)&\geq f(\bx_{k-1};\theta)+\nabla_{\bx}^{\top} f(\bx_{k-1},\theta) (\bx^{\star}-\bx_{k-1})\\
\nabla_{\bx}^{\top} f(\bx_{k-1},\theta) (\bx_{k-1}-\bx^{\star})&\geq f(\bx_{k-1};\theta)-f^\star(\theta),\label{convex}
\end{aligned}
\end{equation}
which is obtained from the convexity of $\bx$ in $f(\bx;\theta)$. Then from the $L_{f,2}$-smooth of $f(\bx;\theta)$ (Lemma 2.3 in \cite{beck2010gradient}), we know,
\begin{align}\label{smooth}
f(\bx_{k};\theta)\leq f(\bx_{k-1};\theta)+\nabla_{\bx}^{\top} f(\bx_{k-1},\theta) (\bx_{k}-\bx_{k-1})+\frac{L_{f,2}}{2}\|\bx_{k}-\bx_{k-1}\|_2^2.
\end{align}
By summing \eqref{convex} and \eqref{smooth}, we obtain,
\[f(\bx_{k};\theta)-f^\star(\theta)\leq\nabla_{\bx}^{\top} f(\bx_{k-1},\theta) (\bx_{k}-\bx^\star)+\frac{L_{f,2}}{2}\|\bx_{k}-\bx_{k-1}\|_2^2.\]
We next use the PGD to give an upper bound on $\nabla_{\bx}^{\top} f(\bx_{k-1},\theta) (\bx_{k}-\bx^\star)$. From the Proposition 6.47 in \cite{bauschke2020correction}, the iteration process of PGD satisfies,
\begin{align*}
\bx_{k-1}-L_{f,2}^{-1}\,\nabla_{\bx} f(
\bx_{k-1};\theta)-\bx_{k}\in\mathcal{N}_{\mathcal{X}}(\bx_{k}).
\end{align*}
As a consequence, from the definition of normal cone $\mathcal{N}_{\mathcal{X}}(\bx_{k})$, we have,
\begin{align*}
\langle \bx_{k-1}-L_{f,2}^{-1}\,\nabla_{\bx} f(
\bx_{k-1};\theta)-\bx_{k},\,\bx_{k}-\bx^\star\rangle&\geq 0\\
L_{f,2}\langle \bx_{k-1}-\bx_{k},\,\bx_{k}-\bx^\star\rangle&\geq \langle\nabla_{\bx} f(
\bx_{k-1};\theta),\,\bx_{k}-\bx^\star\rangle
\end{align*}
So,
\[\frac{L_{f,2}}{2}\|\bx_{k}-\bx_{k-1}\|_2^2+L_{f,2}(\bx_{k-1}-\bx_{k})^{\top}(\bx_{k}-\bx^\star)\geq  f(\bx_{k};\theta)-f^\star(\theta).\]
Finally, denoting $k^\star=\arg\min_{k\in[K]}\Vert\bx_{k}-\bx_{k-1}\Vert_2$, by the Cauchy-Schwartz inequality, we obtain the upper bound of $f(\bx_{k^\star};\theta)-f^\star(\theta)$ as follows.
\begin{equation}
\begin{aligned}
 \frac{L_{f,2}}{2}\|\bx_{k^\star}-\bx_{k^\star-1}\|_2^2+L_{f,2}\|\bx_{k^\star}-\bx_{k^\star-1}\|_2\Vert \bx_{k^\star}-\bx^{\star}\Vert_2 &\geq f(\bx_{k^\star};\theta)-f^\star(\theta)\\
\frac{f(\bx_0;\theta)-f^\star(\theta)}{K}+\sqrt{\frac{2L_{f,2}(f(\bx_0;\theta)-f^\star(\theta))}{ K}} \Vert \bx_{k^\star}-\bx^{\star}\Vert_2 &\overset{(i)}{\geq} f(\bx_{k^\star};\theta)-f^\star(\theta)\\
\frac{f(\bx_0;\theta)-f^\star(\theta)}{K}+2\sqrt{\frac{2L_{f,2}(f(\bx_0;\theta)-f^\star(\theta))}{ K}} c_{\mathcal{X}}R_{\mathcal{X}} &\overset{(ii)}{\geq} f(\bx_{k^\star};\theta)-f^\star(\theta)\label{stage1result}
\end{aligned}
\end{equation}
In $(i)$, we apply \eqref{gua1} on $\|\bx_{k^\star}-\bx_{k^\star-1}\|_2$. Next, by norm equivalence, there exists a constant $c_{\mathcal{X}} > 0$ such that
\begin{align}\label{norm_change}
\|\bx\|_2 \le c_{\mathcal{X}} \|\bx\|, \quad \forall \bx \in \mathcal{X}.
\end{align}
By apply \eqref{norm_change} to the radius of $\mathcal{X}$ under the $\|\cdot\|$ norm (Assumption \ref{a2}), we obtain $(ii)$.

In order to meet the requirement of $(K,\varepsilon)$-$\Theta$-net, from \eqref{stage1result}, we are sufficiently to let,
\begin{align}\label{sufficient_condition_convex}
\frac{f(\bx_0;\theta)-f^\star(\theta)}{K}+2\sqrt{\frac{2L_{f,2}(f(\bx_0;\theta)-f^\star(\theta))}{K}} c_{\mathcal{X}}R_{\mathcal{X}}\leq \varepsilon.\end{align}
Considering the change of variable, we denote $s=\sqrt{\frac{f(\bx_0;\theta)-f^\star(\theta)}{K}}$. Then, \eqref{sufficient_condition_convex} becomes,
\begin{align}
s^2+2c_\mathcal{X}R_\mathcal{X}\sqrt{2L_{f,2}}s\leq \varepsilon\label{stage2quadratic}
\end{align}
When $s=0$, the above quadratic inequality \eqref{stage2quadratic} is satisfied naturally. Thus, by solving \eqref{stage2quadratic} above, \eqref{sufficient_condition_convex} holds is equivalent to, 
\begin{align*}
\sqrt{\frac{f(\bx_0;\theta)-f^\star(\theta)}{K}}&\leq -c_\mathcal{X}R_\mathcal{X}\sqrt{2L_{f,2}}+\sqrt{2L_{f,2}c_\mathcal{X}^2R_\mathcal{X}^2+\varepsilon}\\
f(\bx_0;\theta)-f^\star(\theta)&\leq K\left(-c_\mathcal{X}R_\mathcal{X}\sqrt{2L_{f,2}}+\sqrt{2L_{f,2}c_\mathcal{X}^2R_\mathcal{X}^2+\varepsilon}\right)^2
\end{align*}
Then, by applying Lemma \ref{lem3} again, we derive the sufficient order of $M(K,\varepsilon)$ to guarantee the $(K,\varepsilon)$-$\Theta$-net for a general convex function is shown as follows.
\begin{align*}
f(\bx_0;\theta)-f^\star(\theta)\leq2L_\Theta\Vert \theta-\pi_{\mathcal{D}}(\theta)\Vert &\leq K\left(-c_\mathcal{X}R_\mathcal{X}\sqrt{2L_{f,2}}+\sqrt{2L_{f,2}c_\mathcal{X}^2R_\mathcal{X}^2+\varepsilon}\right)^2\\
2L_\Theta R_{\Theta}M^{-1/d_\Theta}&\lesssim K\left(-c_\mathcal{X}R_\mathcal{X}\sqrt{2L_{f,2}}+\sqrt{2L_{f,2}c_\mathcal{X}^2R_\mathcal{X}^2+\varepsilon}\right)^2\\
M&\gtrsim  \left(\frac{2R_{\Theta}L_\Theta}{K\bigl(\underbrace{-c_\mathcal{X}R_\mathcal{X}\sqrt{2L_{f,2}}+\sqrt{2L_{f,2}c_\mathcal{X}^2R_\mathcal{X}^2+\varepsilon}}_{(\mathrm{I})}\bigr)^2}\right)^{d_\Theta}\\&\overset{(iii)}{\sim} \left(\frac{16L_{f,2}c_\mathcal{X}^2R_\mathcal{X}^2R_{\Theta}L_\Theta}{K\varepsilon^2}\right)^{d_\Theta},
\end{align*}
where $(iii)$ is obtained by considering the first-order Taylor approximation on (I). Hence, we finished the proof of the upper bound $\mathcal{O}(\cdot)$. 

\myparagraph{Complexity lower bound} We consider a family of $p$-growth ($p=\beta$) parameterized functions class, 
\begin{align}\label{polynomialclass}
\underline{f}(\bx;\theta)
= \frac{\min\{L_{f,2},L_\Theta\}}{p(p-1)(R_{\mathcal{X}}+R_\Theta)^{p-2}}\left\Vert \bx - \frac{1}{R_\mathcal{X}+R_\Theta}\begin{bmatrix}\theta \\ \mathbf{0}\end{bmatrix} \right\Vert_2^p,p\in(2,+\infty),
\end{align}
where $\bx \in \mathbb{B}_2(0,R_{\mathcal{X}})\subset\mathbb{R}^{d_{\mathcal{X}}}$,
$\theta \in \mathbb{B}_2(0,R_{\Theta})\subset\mathbb{R}^{d_\Theta}$, 
with $R_{\mathcal{X}} > R_{\Theta}$ (strict) and $R_\mathcal{X}+R_{\Theta}\geq 2.314$. $\begin{bmatrix}\theta \\ \mathbf{0}\end{bmatrix} \in \mathbb{R}^{d_{\mathcal{X}}}$ denotes the vector obtained by zero-padding $\theta$ to dimension $d_{\mathcal{X}}$.
In \eqref{polynomialclass}, if $p \in [1,2]$, PGD enjoys either linear convergence or even finite-step convergence (for $p=1$) \citep{li2018calculus}. 
In contrast, we focus on the function class $p > 2$ with sublinear convergence, where the minimizer
$\bx^\star(\theta) = \frac{1}{R_\mathcal{X}+R_\Theta}\begin{bmatrix}\theta \\ \mathbf{0}\end{bmatrix}$
is well-defined and smooth, while the convergence rate becomes sublinear, which is essential for our analysis. Similar to the proof of Theorem~\ref{thm1}, for any predictor $\pi:\Theta\to\Theta_M$, by assigning $\bx_0=\bx^\star(\pi(\theta))$, we have,
\begin{align*}\underline{f}(\bx_0;\theta)-f^\star(\theta)&=\frac{\min\{L_{f,2},L_\Theta\}}{p(p-1)(R_{\mathcal{X}}+R_\Theta)^{p-2}}\left\Vert  \frac{1}{R_\mathcal{X}+R_\Theta}\begin{bmatrix}\pi(\theta) \\ \mathbf{0}\end{bmatrix}- \frac{1}{R_\mathcal{X}+R_\Theta}\begin{bmatrix}\theta \\ \mathbf{0}\end{bmatrix} \right\Vert_2^p\\&=\frac{\min\{L_{f,2},L_\Theta\}}{p(p-1)(R_{\mathcal{X}}+R_\Theta)^{2p-2}}\left\Vert  \pi(\theta) - \theta  \right\Vert_2^p,\,\forall p\in(2,+\infty).\end{align*}
Thus, minimizing the initial gap $\underline{f}(\bx_0;\theta)-f^\star(\theta)$
is equivalent to solving $\pi_\mathcal{D}(\theta)\in\min_{\theta_m\in\Theta_M} \|\theta_m - \theta\|_2$.
Consequently, the nearest-neighbor predictor $\pi_{\mathcal{D}}$
is the optimal predictor $\pi^\star$ in this setting.
The regularized parameters $p(p-1)(R_{\mathcal{X}}+R_\Theta)^{p-2}$ in the denominator of each function class in \eqref{polynomialclass} are designed to ensure,
Let $\bz := \bx - \frac{1}{R_\mathcal{X}+R_\Theta}\begin{bmatrix}\theta \\ \mathbf{0}\end{bmatrix}$. Then
\[
\nabla_{\bx}^2 \underline{f}(\bx;\theta)
=
\frac{\min\{L_{f,2},L_\Theta\}}{(p-1)(R_{\mathcal{X}}+R_\Theta)^{p-2}}
\left( \|\bz\|_2^{p-2} I + (p-2)\|\bz\|_2^{p-4} \bz\bz^\top \right).
\]
Since the spectral norm of the bracketed term equals $(p-1)\|\bz\|_2^{p-2}$, we obtain
\[
\|\nabla_{\bx}^2 \underline{f}(\bx;\theta)\|_2
\overset{(iv)}{=}
\frac{\min\{L_{f,2},L_\Theta\}}{(R_{\mathcal{X}}+R_\Theta)^{p-2}}
\|\bz\|_2^{p-2}
\le L_{f,2}.
\]
where $(iv)$ from the fact that eigenvalues of a rank-one perturbation of a scalar multiple of the identity can be computed explicitly.
\[\lambda_{\max}
= \left\| \bx - \frac{1}{R_\mathcal{X}+R_{\Theta}}\begin{bmatrix}\theta \\ 0\end{bmatrix} \right\|^{p-2}
+ (p-2)\left\| \bx - \frac{1}{R_\mathcal{X}+R_{\Theta}}\begin{bmatrix}\theta \\ 0\end{bmatrix} \right\|^{p-2}
= (p-1)\left\| \bx - \frac{1}{R_\mathcal{X}+R_{\Theta}}\begin{bmatrix}\theta \\ 0\end{bmatrix} \right\|^{p-2}.\]
And,
\begin{align*}
\|\nabla_\theta \underline{f}(\bx;\theta)\|_2&\leq\frac{\min\{L_{f,2},L_\Theta\}}{p(p-1)(R_{\mathcal{X}}+R_\Theta)^{p-2}}\frac{p}{R_{\mathcal{X}}+R_\Theta}\left(R_\mathcal{X}+\frac{R_\Theta}{R_{\mathcal{X}}+R_\Theta}\right)^{p-1}\\&\leq\frac{\min\{L_{f,2},L_\Theta\}}{p(p-1)(R_{\mathcal{X}}+R_\Theta)^{p-2}}\frac{p}{R_{\mathcal{X}}+R_\Theta}(R_\mathcal{X}+R_\Theta)^{p-1}\\
&\leq\frac{\min\{L_{f,2},L_\Theta\}}{p-1}\leq L_\Theta,
\end{align*}
the last step holds because $p>2$. Then, since
\begin{align}\label{radiusbound}
\sup_{\bx\in\mathbb{B}_2(0,R_{\mathcal{X}}),\,\theta\in\mathbb{B}_2(0,R_{\Theta})}
\left\Vert \bx - \frac{1}{R_\mathcal{X}+R_{\Theta}}\begin{bmatrix}\theta \\ \mathbf{0}\end{bmatrix} \right\Vert_2
\leq R_{\mathcal{X}} + R_{\Theta},
\end{align}
we obtain $(v)$. The normalization factor $p(p-1)(R_{\mathcal{X}}+R_{\Theta})^{p-2}$ in the denominator of \eqref{polynomialclass} is introduced to ensure that, for each $p$, the function class is $L_{f,2}$-smooth (Assumption~\ref{a3}); in particular, it exactly offsets the worst-case growth of the Hessian at the boundary of the domain $\mathbb{B}_2(0,R_\mathcal{X})$.

For the parameterized function class \eqref{polynomialclass}, the resulting PGD dynamics exhibit sublinear convergence. Since this convex family encompasses both strongly convex functions and functions satisfying a second-order growth condition, the associated complexity lower bound is expected to be no smaller than that in Theorem~\ref{thm1}. In this proof, we characterize the lower bound within this boarder class.

By fixing a $p\in(2,+\infty)$, we firstly derive the gradient on $\underline{f}(\bx;\theta)$,
\[\nabla_{\bx}\underline{f}(\bx;\theta)=\frac{\min\{L_{f,2},L_\Theta\}}{(p-1)(R_\mathcal{X}+R_\Theta)^{p-2}}\left\Vert\bx-\frac{1}{R_\mathcal{X}+R_{\Theta}}\begin{bmatrix}\theta \\ \mathbf{0}\end{bmatrix}\right\Vert_2^{p-2}\left(\bx-\frac{1}{R_\mathcal{X}+R_{\Theta}}\begin{bmatrix}\theta \\ \mathbf{0}\end{bmatrix}\right).\]
So with the definition of PGD in Definition \ref{dfn2}, with the centralization of $\frac{1}{R_\mathcal{X}+R_{\Theta}}\begin{bmatrix}\theta \\ \mathbf{0}\end{bmatrix}$, denoting $L_{f,\Theta}=\frac{\min\{L_{f,2},L_\Theta\}}{L_{f,2}}\leq1$, the gradient descent process can be described as follows,
\begin{align*}
\bx_{k+1}&=\bx_{k}-L_{f,2}^{-1}\nabla_{\bx}\underline{f}(\bx_k;\theta)\\
\bx_{k+1}-\frac{1}{R_\mathcal{X}+R_{\Theta}}\begin{bmatrix}\theta \\ \mathbf{0}\end{bmatrix}&=\bx_{k}-\frac{1}{R_\mathcal{X}+R_{\Theta}}\begin{bmatrix}\theta \\ \mathbf{0}\end{bmatrix}-L_{f,2}^{-1}\nabla_{\bx}\underline{f}(\bx_k;\theta)\\
\bx_{k+1}-\frac{1}{R_\mathcal{X}+R_{\Theta}}\begin{bmatrix}\theta \\ \mathbf{0}\end{bmatrix}&=\left(1-\frac{L_{f,\Theta}(R_{\mathcal{X}}+R_\Theta)^{2-p}}{p-1}\left\Vert\bx_k-\frac{1}{R_\mathcal{X}+R_{\Theta}}\begin{bmatrix}\theta \\ \mathbf{0}\end{bmatrix}\right\Vert_2^{p-2}\right)\left(\bx_k-\frac{1}{R_\mathcal{X}+R_{\Theta}}\begin{bmatrix}\theta \\ \mathbf{0}\end{bmatrix}\right)\\
\left\|\bx_{k+1}-\frac{1}{R_\mathcal{X}+R_{\Theta}}\begin{bmatrix}\theta \\ \mathbf{0}\end{bmatrix}\right\|_2&=\underbrace{\left|1-\frac{L_{f,\Theta}(R_{\mathcal{X}}+R_\Theta)^{2-p}}{p-1}\left\Vert\bx_k-\frac{1}{R_\mathcal{X}+R_{\Theta}}\begin{bmatrix}\theta \\ \mathbf{0}\end{bmatrix}\right\Vert_2^{p-2}\right|}_{\mathrm{(II)}}\left\Vert\bx_k-\frac{1}{R_\mathcal{X}+R_{\Theta}}\begin{bmatrix}\theta \\ \mathbf{0}\end{bmatrix}\right\Vert_2.
\end{align*}
In the term (II), $0<1-\frac{L_{f,\Theta}(R_{\mathcal{X}}+R_\Theta)^{2-p}}{p-1}\left\Vert\bx_k-\frac{1}{R_\mathcal{X}+R_{\Theta}}\begin{bmatrix}\theta \\ \mathbf{0}\end{bmatrix}\right\Vert_2^{p-2}<1$ because $\left\Vert\bx_k-\frac{1}{R_\mathcal{X}+R_{\Theta}}\begin{bmatrix}\theta \\ \mathbf{0}\end{bmatrix}\right\Vert_2\leq R_\Theta$, ($\left\Vert\bx_0-\frac{1}{R_\mathcal{X}+R_{\Theta}}\begin{bmatrix}\theta\\ \mathbf{0}\end{bmatrix}\right\Vert_2\leq R_\Theta$). Thus, from $L_{f,\Theta}\leq 1$, we have,
\[\frac{L_{f,\Theta}(R_{\mathcal{X}}+R_\Theta)^{2-p}}{p-1}\left\Vert\bx_k-\frac{1}{R_\mathcal{X}+R_{\Theta}}\begin{bmatrix}\theta \\ \mathbf{0}\end{bmatrix}\right\Vert_2^{p-2}\leq \frac{L_{f,\Theta}}{p-1}\left(\frac{R_\Theta}{R_\Theta+R_\mathcal{X}}\right)^{p-2}< 1.\]
Moreover, it is worth noting that, in our example \eqref{polynomialclass}, PGD reduces to standard gradient descent. This is because both the warm-start $\bx_0$ and the optimal solution $\frac{1}{R_\mathcal{X}+R_{\Theta}}\begin{bmatrix}\theta \\ \mathbf{0}\end{bmatrix}$ lie in the same set $ \mathbb{B}_2(0,R_{\Theta})$ from $R_\mathcal{X}+R_{\Theta}\geq 2.314$, and the iterates never leave the region.

The next part of proof is given in two folds.

\myparagraph{Case $p\geq3$ \& $p\to2^+$.} Next, we have to use the following Bernoulli inequality. For $p\ge3$ and $0<s<\frac{1}{p-2}$, then,
\begin{equation}
\begin{aligned}\label{Bernouli}
(1 - s)^{p-2} &\ge 1 - (p - 2)s\\
(1 - s)^{-(p-2)} &\le \frac{1}{1 - (p - 2)s}\\
&\overset{(vi)}{=}1 + \frac{(p - 2)s}{1 - (p - 2)s}.
\end{aligned}
\end{equation}
We employ a bound on $s$ as $0<s\leq\frac{1}{2(p-2)}$ to avoid the RHS of $(vi)$ blow up to $\infty$.
Let 
$s = \frac{L_{f,\Theta}(R_{\mathcal{X}}+R_\Theta)^{2-p}}{p-1}
\left\Vert \bx_k - \frac{1}{R_\mathcal{X}+R_{\Theta}}\begin{bmatrix}\theta \\ \mathbf{0}\end{bmatrix} \right\Vert_2^{p-2}$,
to apply the Bernoulli inequality \eqref{Bernouli}, we invoke \[\left\Vert\bx_k-\frac{1}{R_\mathcal{X}+R_{\Theta}}\begin{bmatrix}\theta \\ \mathbf{0}\end{bmatrix}\right\Vert_2\leq\left\Vert\bx_0-\frac{1}{R_\mathcal{X}+R_{\Theta}}\begin{bmatrix}\theta \\ \mathbf{0}\end{bmatrix}\right\Vert_2=\frac{1}{R_\mathcal{X}+R_{\Theta}}\left\Vert\begin{bmatrix}\pi_\mathcal{D}(\theta) \\ \mathbf{0}\end{bmatrix}-\begin{bmatrix}\theta \\ \mathbf{0}\end{bmatrix}\right\Vert_2\leq 0.865R_\Theta,\] which yields
\begin{align}\label{applybernoulli}&\left\Vert\bx_k-\frac{1}{R_\mathcal{X}+R_{\Theta}}\begin{bmatrix}\theta \\ \mathbf{0}\end{bmatrix}\right\Vert_2\leq 0.865R_\Theta\overset{(vii)}{\leq}(R_\mathcal{X}+R_\Theta)\left(\frac{p-1}{2L_{f,\Theta}(p-2)}\right)^{\frac{1}{p-2}}\nonumber\\\rightarrow& 0<\frac{(R_{\mathcal{X}}+R_\Theta)^{2-p}}{p-1}
\left\Vert \bx_k - \frac{1}{R_\mathcal{X}+R_{\Theta}}\begin{bmatrix}\theta \\ \mathbf{0}\end{bmatrix} \right\Vert_2^{p-2}<\frac{1}{2(p-2)},
\end{align}
where $(vii)$ comes from using $L_{f,\Theta}\leq1$, 
\begin{align*}\inf_{p\in(2,+\infty)}\left(\frac{p-1}{2p-4}\right)^{\frac{1}{p-2}}\left(\frac{1}{L_{f,\Theta}}\right)^{\frac{1}{p-2}}&\geq\inf_{p\in(2,+\infty)}\left(\frac{p-1}{2p-4}\right)^{\frac{1}{p-2}}\inf_{p\in(2,+\infty)}\left(\frac{1}{L_{f,\Theta}}\right)^{\frac{1}{p-2}}\\&=0.865\lim_{p\to\infty}\left(\frac{1}{L_{f,\Theta}}\right)^{\frac{1}{p-2}}=0.865.
\end{align*}
Besides, when $p\to 2^+$, \eqref{Bernouli} becomes,
\[(1 - s)^{p-2} \approx 1 - (p - 2)s\rightarrow (1 - s)^{-(p-2)}\approx 1 + \frac{(p - 2)s}{1 - (p - 2)s},\,\forall s\in(0,1).\]
Then for both cases $p\geq 3$ and $p\to2^+$, we have,
\begin{equation}
\begin{aligned}\label{bernoullibound}
&\left(1-\frac{L_{f,\Theta}(R_\mathcal{X}+R_\Theta)^{2-p}}{p-1}\left\Vert\bx_k-\frac{1}{R_\mathcal{X}+R_{\Theta}}\begin{bmatrix}\theta \\ \mathbf{0}\end{bmatrix}\right\Vert_2^{p-2}\right)^{-(p-2)}-1\\\overset{(viii)}{\leq}& \frac{L_{f,\Theta}(R_\mathcal{X}+R_\Theta)^{2-p}\frac{p - 2}{p-1}\left\Vert\bx_k-\frac{1}{R_\mathcal{X}+R_{\Theta}}\begin{bmatrix}\theta \\ \mathbf{0}\end{bmatrix}\right\Vert_2^{p-2}}{1 - L_{f,\Theta}(R_\mathcal{X}+R_\Theta)^{2-p}\frac{p - 2}{p-1}\left\Vert\bx_k-\frac{1}{R_\mathcal{X}+R_{\Theta}}\begin{bmatrix}\theta \\ \mathbf{0}\end{bmatrix}\right\Vert_2^{p-2}}\\
\overset{(ix)}{\leq}& {2L_{f,\Theta}(R_\mathcal{X}+R_\Theta)^{2-p}\frac{p - 2}{p-1}\left\Vert\bx_k-\frac{1}{R_\mathcal{X}+R_{\Theta}}\begin{bmatrix}\theta \\ \mathbf{0}\end{bmatrix}\right\Vert_2^{p-2}},
\end{aligned}
\end{equation}
where $(viii)$ is obtained from applying \eqref{Bernouli} to linearize the LHS term. Then, by consider the condition \eqref{applybernoulli}, we get $(ix)$. As a consequence, we consider the gap between two consecutive steps in the PGD sequence, $\left\Vert\bx_{k+1}-\frac{1}{R_\mathcal{X}+R_{\Theta}}\begin{bmatrix}\theta \\ \mathbf{0}\end{bmatrix}\right\Vert_2^{-(p-2)}-\left\Vert\bx_k-\frac{1}{R_\mathcal{X}+R_{\Theta}}\begin{bmatrix}\theta \\ \mathbf{0}\end{bmatrix}\right\Vert_2^{-(p-2)}>0$, then,
\begin{equation}
\begin{aligned}\label{sequentialbound}
&\left\Vert\bx_{k+1}-\frac{1}{R_\mathcal{X}+R_{\Theta}}\begin{bmatrix}\theta \\ \mathbf{0}\end{bmatrix}\right\Vert_2^{-(p-2)}-\left\Vert\bx_k-\frac{1}{R_\mathcal{X}+R_{\Theta}}\begin{bmatrix}\theta \\ \mathbf{0}\end{bmatrix}\right\Vert_2^{-(p-2)}\\=&\left\Vert\bx_k-\frac{1}{R_\mathcal{X}+R_{\Theta}}\begin{bmatrix}\theta \\ \mathbf{0}\end{bmatrix}\right\Vert_2^{-(p-2)}\underbrace{\left(\left(1-\frac{L_{f,\Theta}(R_{\mathcal{X}}+R_\Theta)^{2-p}}{p-1}\left\Vert\bx_k-\frac{1}{R_\mathcal{X}+R_{\Theta}}\begin{bmatrix}\theta \\ \mathbf{0}\end{bmatrix}\right\Vert_2^{p-2}\right)^{-(p-2)}-1\right)}_{\mathrm{(III)}}\\
\overset{(x)}{\leq}&\left\Vert\bx_k-\frac{1}{R_\mathcal{X}+R_{\Theta}}\begin{bmatrix}\theta \\ \mathbf{0}\end{bmatrix}\right\Vert_2^{-(p-2)}2L_{f,\Theta}(R_{\mathcal{X}}+R_\Theta)^{2-p}\frac{p - 2}{p-1}\left\Vert\bx_k-\frac{1}{R_\mathcal{X}+R_{\Theta}}\begin{bmatrix}\theta \\ \mathbf{0}\end{bmatrix}\right\Vert_2^{p-2}\\
=&2L_{f,\Theta}(R_{\mathcal{X}}+R_\Theta)^{2-p}\frac{p-2}{p-1},
\end{aligned}
\end{equation}
where $(x)$ is derived from using \eqref{bernoullibound} on (III). By summing from $0$ to $K-1$, consider the telescoping sum, we have:
\begin{align*}
&\min_{k=1,...,K}\underline{f}(\bx_k;\theta)-\underline{f}^\star(\theta)\\=&\underline{f}(\bx_K;\theta)-\underline{f}^\star(\theta)\\=&\frac{\min\{L_{f,2},L_\Theta\}}{p(p-1) (R_\mathcal{X}+R_\Theta)^{p-2}}\left\Vert\bx_K-\frac{1}{R_\mathcal{X}+R_{\Theta}}\begin{bmatrix}\theta \\ \mathbf{0}\end{bmatrix}\right\Vert_2^p\\
=&\frac{\min\{L_{f,2},L_\Theta\}}{p(p-1) (R_\mathcal{X}+R_\Theta)^{p-2}}\left(\left\Vert\bx_K-\frac{1}{R_\mathcal{X}+R_{\Theta}}\begin{bmatrix}\theta \\ \mathbf{0}\end{bmatrix}\right\Vert_2^{-(p-2)}\right)^{-\frac{p}{p-2}}\\=&\frac{\min\{L_{f,2},L_\Theta\}}{p(p-1) (R_\mathcal{X}+R_\Theta)^{p-2}}\left(\sum_{k=0}^{K-1}\underbrace{\left(\left\Vert\bx_{k+1}-\frac{1}{R_\mathcal{X}+R_{\Theta}}\begin{bmatrix}\theta \\ \mathbf{0}\end{bmatrix}\right\Vert_2^{-(p-2)}-\left\Vert\bx_{k}-\frac{1}{R_\mathcal{X}+R_{\Theta}}\begin{bmatrix}\theta \\ \mathbf{0}\end{bmatrix}\right\Vert_2^{-(p-2)}\right)}_{\mathrm{(IV)}}\right.\\&+\left\Vert\bx_{0}-\frac{1}{R_\mathcal{X}+R_{\Theta}}\begin{bmatrix}\theta \\ \mathbf{0}\end{bmatrix}\right\Vert_2^{-(p-2)})^{-\frac{p}{p-2}}
\\\overset{(xi)}{\geq}&\frac{\min\{L_{f,2},L_\Theta\}}{p(p-1) (R_\mathcal{X}+R_\Theta)^{p-2}}\left(\left\Vert\bx_0-\frac{1}{R_\mathcal{X}+R_{\Theta}}\begin{bmatrix}\theta \\ \mathbf{0}\end{bmatrix}\right\Vert_2^{-(p-2)}+2KL_{f,\Theta}(R_\mathcal{X}+R_\Theta)^{2-p}\frac{p - 2}{p-1}\right)^{-\frac{p}{p-2}},
\end{align*}
where $(xi)$ is from applying \eqref{sequentialbound} on totally $K$ consecutive difference terms. To establish the complexity lower bound, it suffices to identify at least one 
$\theta \in \mathbb{B}_2(0,R_\Theta)$ such that, even with the optimal warm-start 
$\bx_0 = \pi^\star(\theta)$, the following holds:
\begin{align*}
\frac{\min\{L_{f,2},L_\Theta\}}{p(p-1) (R_\mathcal{X}+R_\Theta)^{p-2}}\left(\left\Vert\bx_0-\frac{1}{R_\mathcal{X}+R_{\Theta}}\begin{bmatrix}\theta \\ \mathbf{0}\end{bmatrix}\right\Vert_2^{-(p-2)}+2L_{f,\Theta}K(R_\mathcal{X}+R_\Theta)^{2-p}\frac{p - 2}{p-1}\right)^{-\frac{p}{p-2}}&>\varepsilon
\end{align*}
Then, by rearranging terms above, we obtain,
{\small{\begin{align*}
\left(\frac{\varepsilon p(p-1) (R_\mathcal{X}+R_\Theta)^{p-2}}{\min\{L_{f,2},L_\Theta\}}\right)^{-\frac{p-2}{p}}&>\left\Vert\bx_0-\frac{1}{R_\mathcal{X}+R_{\Theta}}\begin{bmatrix}\theta \\ \mathbf{0}\end{bmatrix}\right\Vert_2^{-(p-2)}+2L_{f,\Theta}(R_\mathcal{X}+R_\Theta)^{2-p}\frac{p - 2}{p-1}K\\
\left\Vert\bx_0-\frac{1}{R_\mathcal{X}+R_{\Theta}}\begin{bmatrix}\theta \\ \mathbf{0}\end{bmatrix}\right\Vert_2^{-(p-2)}&<\left(\frac{\varepsilon p(p-1) (R_\mathcal{X}+R_\Theta)^{p-2}}{\min\{L_{f,2},L_\Theta\}}\right)^{-\frac{p-2}{p}}-2L_{f,\Theta}(R_\mathcal{X}+R_\Theta)^{2-p}\frac{p - 2}{p-1}K\\
\left\Vert\bx_0-\frac{1}{R_\mathcal{X}+R_{\Theta}}\begin{bmatrix}\theta \\ \mathbf{0}\end{bmatrix}\right\Vert_2&>(R_\mathcal{X}+R_\Theta)\left(\left(\frac{p(p-1)\varepsilon}{\min\{L_{f,2},L_\Theta\}}\right)^{-\frac{p-2}{p}}(R_{\mathcal{X}}+R_\Theta)^\frac{2(p-2)}{p}-2L_{f,\Theta}\frac{p - 2}{p-1}K\right)^{-\frac{1}{p-2}}\\
\left\Vert\pi_\mathcal D(\theta)-\theta \right\Vert_2&>(R_\mathcal{X}+R_\Theta)^2\left(\left(\frac{p(p-1)\varepsilon}{\min\{L_{f,2},L_\Theta\}}\right)^{-\frac{p-2}{p}}(R_{\mathcal{X}}+R_\Theta)^\frac{2(p-2)}{p}-2L_{f,\Theta}\frac{p - 2}{p-1}K\right)^{-\frac{1}{p-2}}.
\end{align*}}}
It is notable that for any fixed online-budget $K$ and error tolerance $\varepsilon$, we want to find an optimal order $p$ to let the tightest bound hold,
\begin{align}\label{optimalp}
p^\star(K,\varepsilon)=\arg\min_{p\in[3,\bar{p}]}\left(\left(p(p-1)\varepsilon\right)^{-\frac{p-2}{p}}(\min\{L_{f,2},L_\Theta\}(R_{\mathcal{X}}+R_\Theta)^2)^\frac{p-2}{p}-2L_{f,\Theta}\frac{p - 2}{p-1}K\right)^{-\frac{1}{p-2}}.
\end{align}
The upper bound order $\bar p$ in \eqref{optimalp} is required to satisfy the following
condition next:
\begin{itemize}
\item For a given online oracle cost $K$,
\begin{align*}
\left(\bar{p}(\bar{p}-1)\varepsilon\right)^{-\frac{\bar{p}-2}{\bar{p}}}(\min\{L_{f,2},L_\Theta\}(R_{\mathcal{X}}+R_\Theta)^2)^\frac{\bar{p}-2}{\bar{p}}-2\frac{\bar{p} - 2}{\bar{p}-1}K
&> 0\\
\left(\frac{\min\{L_{f,2},L_\Theta\}(R_{\mathcal{X}}+R_\Theta)^2}{\bar{p}(\bar{p}-1)\varepsilon}\right)^\frac{\bar{p}-2}{\bar{p}}
&> 2\frac{\bar{p} - 2}{\bar{p}-1}K\\
\sqrt{\frac{\min\{L_{f,2},L_\Theta\}(R_{\mathcal{X}}+R_\Theta)^2}{K\varepsilon}}&\overset{}{\gtrsim} \bar{p}
\end{align*}
\end{itemize}
So,
\[\bar{p}\leq\mathcal{O}\left(\sqrt{\frac{\min\{L_{f,2},L_\Theta\}(R_{\mathcal{X}}+R_\Theta)^2}{K\varepsilon}}\right).\]
The analytical solution of \eqref{optimalp} is extremely difficult to find. We can only provide the lower bound of the memory complexity $M(K,\varepsilon)$ for each $p\geq3$ and $p\to2^+$ as follows,
\begin{equation}
\begin{aligned}\label{lbthm3}
&\Omega\left(\left[\left(\frac{p(p-1)\varepsilon}{\min\{L_{f,2},L_\Theta\}}\right)^{-\frac{p-2}{p}}(R_{\mathcal{X}}+R_\Theta)^\frac{2(p-2)}{p}-\frac{2\min\{L_{f,2},L_\Theta\}}{L_{f,2}}\frac{p - 2}{p-1}K\right]_+^{\frac{d_\Theta}{p-2}}\right)\\\overset{(xii)}{\geq}&\Omega\left(\left[\left(\frac{p(p-1)\varepsilon}{\min\{L_{f,2},L_\Theta\}}\right)^{-\frac{p-2}{p}}R_\Theta^\frac{2(p-2)}{p}-\frac{2\min\{L_{f,2},L_\Theta\}}{L_{f,2}}\frac{p - 2}{p-1}K\right]_+^{\frac{d_\Theta}{p-2}}\right)\\\overset{(xiii)}{=}&\Omega\left(R_\Theta^\frac{d_\Theta}p\left[\left(\frac{\min\{L_{f,2},L_\Theta\}R_\Theta}{\varepsilon}\right)^{\frac{p-2}{p}}-\frac{2\min\{L_{f,2},L_\Theta\}p(p-2)}{L_{f,2}R_\Theta^\frac{p-2}{p}}\underbrace{\left(\frac{1}{p(p-1)}\right)^{\frac{2}{p}}}_{\text{(V)}}K\right]_+^{\frac{d_\Theta}{p-2}}\right)\\\overset{(xiv)}{\geq}&\Omega\left(R_\Theta^\frac{d_\Theta}p\left[\left(\frac{\min\{L_{f,2},L_\Theta\}R_\Theta}{\varepsilon}\right)^{\frac{p-2}{p}}-\frac{2\min\{L_{f,2},L_\Theta\}p(p-2)}{L_{f,2}R_\Theta^\frac{p-2}{p}}K\right]_+^{\frac{d_\Theta}{p-2}}\right)\\
=&\Omega\left((\min\{L_{f,2},L_\Theta\}R_\Theta^2)^{\frac{d_\Theta}{p}}\left[\left(\frac{1}{\varepsilon}\right)^{\frac{p-2}{p}}-\frac{2\min\{L_{f,2},L_\Theta\}\beta(\beta-2)}{(\min\{L_{f,2},L_\Theta\}R_\Theta^2)^{\frac{p-2}{p}}L_{f,2}}K\right]_+^{\frac{d_\Theta}{p-2}}\right)
\end{aligned}
\end{equation}
where $(xii)$ follows from the fact that $R_{\mathcal{X}} + R_\Theta > R_\Theta$.
Step $(xiii)$ is obtained by factoring out $\left(\frac{1}{p(p-1)}\right)^{\frac{d_\Theta}{p}} \sim \Theta(1),\,\forall p\in(2,+\infty)$.
Finally, noting that $\left(\frac{1}{p(p-1)}\right)^{\frac{2}{p}} < 1$, we arrive at the expression in $(xiv)$. Hence, by denoting $\beta$ as $p$, we finished the proof of the first part.

\myparagraph{Case $p\in(2,3)$.} In this case, \eqref{Bernouli} no longer holds anymore. We shall find another way to induce a similar intermediate result like \eqref{bernoullibound}. So, we derive, for all $s\in (0,1)$,
\[(1-s)^{-(p-2)}-1=\int_0^s(p-2)(1-t)^{1-p}dt=\int_0^s(p-2)(1-t)^{1-p}dt\leq(p-2)s(1-s)^{1-p}.\]
By denoting $s=\frac{L_{f,\Theta}(R_{\mathcal{X}}+R_\Theta)^{2-p}}{p-1}
\left\Vert \bx_k - \frac{1}{R_\mathcal{X}+R_{\Theta}}\begin{bmatrix}\theta \\ \mathbf{0}\end{bmatrix} \right\Vert_2^{p-2}<1$, similar to the argument in \eqref{applybernoulli}, we have,
\[(1-s)^{1-p}\leq \left(1-\frac{L_{f,\Theta}}{p-1}\left(\frac{2R_\Theta}{R_\Theta+R_\mathcal{X}}\right)^{p-2}\right)^{-2}.\]
Thus, for given $p\in(2,3)$ and not close to $2^+$, $\left(1-\frac{L_{f,\Theta}}{p-1}\left(\frac{2R_\Theta}{R_\Theta+R_\mathcal{X}}\right)^{p-2}\right)^{-2}\sim\Theta(1)$. Hence, similar to \eqref{bernoullibound}, we have,
\begin{align*}
&\left(1-\frac{L_{f,\Theta}(R_\mathcal{X}+R_\Theta)^{2-p}}{p-1}\left\Vert\bx_k-\frac{1}{R_\mathcal{X}+R_{\Theta}}\begin{bmatrix}\theta \\ \mathbf{0}\end{bmatrix}\right\Vert_2^{p-2}\right)^{-(p-2)}-1\\{\leq}& {c_{\mathrm{bound}}L_{f,\Theta}(R_\mathcal{X}+R_\Theta)^{2-p}\frac{p - 2}{p-1}\left\Vert\bx_k-\frac{1}{R_\mathcal{X}+R_{\Theta}}\begin{bmatrix}\theta \\ \mathbf{0}\end{bmatrix}\right\Vert_2^{p-2}},
\end{align*}
where $c_{\mathrm{bound}}=\left(1-\frac{L_{f,\Theta}}{p-1}\left(\frac{2R_\Theta}{R_\Theta+R_\mathcal{X}}\right)^{p-2}\right)^{-2}\le\left(1-\frac{L_{f,\Theta}}{p-1}\left(0.865R_\Theta\right)^{p-2}\right)^{-2}$. So, by exact the same proving steps, for each $p$, we obtain,
\[\left\Vert\pi_\mathcal D(\theta)-\theta \right\Vert_2>(R_\mathcal{X}+R_\Theta)^2\left(\left(\frac{p(p-1)\varepsilon}{\min\{L_{f,2},L_\Theta\}}\right)^{-\frac{p-2}{p}}(R_{\mathcal{X}}+R_\Theta)^\frac{2(p-2)}{p}-c_{\mathrm{bound}}L_{f,\Theta}\frac{p - 2}{p-1}K\right)^{-\frac{1}{p-2}}.\]
Therefore, the lower bound of the memory complexity $M(K,\varepsilon)$ is,
\begin{align*}M(K,\varepsilon)&\ge\Omega\left(\left[\left(\frac{\min\{L_{f,2},L_\Theta\}R_\Theta^2}{\varepsilon}\right)^{\frac{p-2}{p}}-\frac{c_{\mathrm{bound}}\min\{L_{f,2},L_\Theta\}p(p-2)}{L_{f,2}}K\right]_+^{\frac{d_\Theta}{p-2}}\right)\\&=\Omega\left((\min\{L_{f,2},L_\Theta\}R_\Theta^2)^{\frac{d_\Theta}{p}}\left[\left(\frac{1}{\varepsilon}\right)^{\frac{p-2}{p}}-\frac{c_{\mathrm{bound}}\min\{L_{f,2},L_\Theta\}\beta(\beta-2)}{(\min\{L_{f,2},L_\Theta\}R_\Theta^2)^{\frac{p-2}{p}}L_{f,2}}K\right]_+^{\frac{d_\Theta}{p-2}}\right).\end{align*}
The last step is obtained by taking $\beta=p$. Hence, we finished the proof.
\end{proof}
\section{Proof of Theorem \ref{thm2}}\label{pfthm2}
The complete version of Theorem \ref{thm2} is shown as follows.
\begin{thm}[Memory complexity for $\beta$-growth problems]\label{thm4}
Under Assumptions~\ref{a1}--\ref{a4} and~\ref{a7}, suppose that $\beta>2$ and fix the number of online oracle cost $K$. Then the memory complexity required to achieve accuracy $\varepsilon$ satisfies the following upper and lower bounds.

\textbf{(Upper bound).}
\begin{align}\label{betaub1}
M(K,\varepsilon)
\leq
\mathcal{O}\!\left(
(L_{\Theta}R_{\Theta})^{d_\Theta}\left[
\left(\frac{1}{\varepsilon}\right)^{\frac{\beta-2}{\beta}}
- C_1 K
\right]_+^{\frac{\beta d_\Theta}{\beta-2}}
\right).
\end{align}

\textbf{(Lower bound).}
If $\Phi \le \frac{\min\{L_{f,2},L_\Theta\}}{\beta(\beta-1) (R_{\mathcal{X}}+R_\Theta)^{\beta-2}}$, then
\begin{align}\label{betalb1}
M(K,\varepsilon)
\geq
\Omega\!\left(
(\min\{L_{f,2},L_{\Theta}\}R_{\Theta}^2)^{\frac{d_\Theta}\beta}\left[
\left(\frac{1}{\varepsilon}\right)^{\frac{\beta-2}{\beta}}
- C_2 K
\right]_+^{\frac{d_\Theta}{\beta-2}}
\right),
\end{align}
where $[\cdot]_+ := \max\{\cdot,0\}$. And, \[
C_1
=
\frac{(\beta-2)\Phi^{\frac{2}{\beta}}}
{16\beta L_{f,2}},\,
C_2
=
\frac{c_{\mathrm{bound}}\beta(\beta-2)\min\{L_{f,2},L_\Theta\}}
{(\min\{L_{\Theta},L_{f,2}\}R_{\Theta}^2)^{\frac{\beta-2}{\beta}}L_{f,2}},
\]
where
$c_{\mathrm{bound}} =
\left(1-\frac{L_{f,\Theta}}{\beta-1}\left(0.865R_\Theta\right)^{\beta-2}\right)^{-2}$
when $\beta \in (2,3)$, and
$
c_{\mathrm{bound}} = 2$
otherwise.
\end{thm}
\begin{proof}
In the Theorem 4.iii of \cite{frankel2015splitting}, Frankel et al. give the converge guarantee of any descent algorithm under prescribed KL-conditions. We first apply their result to our PGD in the following lemma.
\begin{lem}
\citep{frankel2015splitting} In case of $\beta$-growth condition when $\beta\in (2,+\infty)$, if $f(\bx_0;\theta)-f^\star(\theta)\leq \left(\frac{\Phi^{2/\beta}(\beta-2)}{16\beta L_{f,2}\left(2^{\frac{2-\beta}{2-2\beta}}-1\right)}\right)^{-\frac{\beta}{\beta-2}}$, then,
\begin{align}
\min_{k\in[K]}f(\bx_k;\theta)-f^\star(\theta) \;\le\;
\left(\frac{(\beta-2)\Phi^{2/\beta}}{16\beta L_{f,2}}K+(f(\bx_0;\theta)-f^\star(\theta))^{-\frac{\beta-2}{\beta}}\right)^{-\frac{\beta}{\beta-2}}.\label{klgua}
\end{align}
\label{lem2}
\end{lem}
\begin{proof}
The proof of this lemma heavily relies on the analyzing framework in \cite{frankel2015splitting}. In this proof, we make use of the notation in their paper. First, respect to the hypotheses $\mathbf{H_1}$ and $\mathbf{H_2}$ in \cite{frankel2015splitting}, we define the sufficient decreasing factor $a_k$ and relative error factor $b_k$ as follows. For series $a_k$, considering the descent lemma of PGD (Lemma 2.6 in \cite{beck2010gradient}), we have,
\[f(\bx_{k};\theta)\leq f(\bx_{k-1};\theta)-\frac{L_{f,2}}{2}\|\bx_{k}-\bx_{k-1}\|_2^2,\,k\geq1.\]
 So, $a_k=\frac{L_{f,2}}{2}>0$ for all $k$. Besides, for the calculation of series $b_k$, we obtain,
 \begin{align}\label{subdiffeq}
\partial_\bx(f(\bx_{k+1};\theta)+\mathbf{1}_{\mathcal{X}}(\bx_{k+1}))=\nabla_\bx f(\bx_{k+1};\theta)+\mathcal{N}_{\mathcal{X}}(\bx_{k+1}),
 \end{align}
where $\mathcal{N}_{\mathcal{X}}(\bx_{k+1})$ is the normal cone of $\mathcal{X}$ at $\bx_{k+1}$. And by Proposition 6.47 in \cite{bauschke2020correction} again, we have,
\begin{equation}
\begin{aligned}\label{subdiffeq1}
\bx_{k}-L_{f,2}^{-1}\,\nabla_{\bx} f(
\bx_{k};\theta)-\bx_{k+1}\in\mathcal{N}_{\mathcal{X}}(\bx_{k+1})\\
L_{f,2}(\bx_{k}-\bx_{k+1})-\nabla_{\bx} f(
\bx_{k};\theta)\in\mathcal{N}_{\mathcal{X}}(\bx_{k+1})
\end{aligned}
\end{equation}
So, by adding $\nabla_\bx(f(\bx_{k+1};\theta)$ on both sides of \eqref{subdiffeq1} and using \eqref{subdiffeq}, we get,
\[L_{f,2}(\bx_{k}-\bx_{k+1})+(\nabla_\bx(f(\bx_{k+1};\theta)-\nabla_\bx(f(\bx_{k};\theta))\in\partial_\bx(f(\bx_{k+1};\theta)+\mathbf{1}_{\mathcal{X}}(\bx_{k+1})).\]
Then taking the Euclidean norm on both sides, with the triangular inequality, we have
\begin{align*}
L_{f,2}\|\bx_{k+1}-\bx_k\|_2+\underbrace{\|\nabla_\bx(f(\bx_{k+1};\theta)-\nabla_\bx(f(\bx_{k};\theta))\|_2}_{\mathrm{(VI)}}&\geq\Vert \partial_\bx(f(\bx_{k+1};\theta)+\mathbf{1}_{\mathcal{X}}(\bx_{k+1}))\|_{2-}\\
2L_{f,2}\|\bx_{k+1}-\bx_k\|_2&\overset{(xv)}{\geq}\Vert \partial_\bx(f(\bx_{k+1};\theta)+\mathbf{1}_{\mathcal{X}}(\bx_{k+1}))\|_{2-},
\end{align*}
where $\Vert \partial_\bx(f(\bx;\theta)+\mathbf{1}_{\mathcal{X}}(\bx))\|_{2-}=\inf_{\bz\in\partial_\bx(f(\bx;\theta)+\mathbf{1}_{\mathcal{X}}(\bx))}\|\bz\|_2$. $(xv)$ is obainted from the $L_{f,2}$-smoothness (gradient Lipschitz) on term (VI). Hence, we have $b_k=\frac{1}{2L_{f,2}}$ for all $k$. In \cite{frankel2015splitting}, an auxiliary function $\varphi(t)=\frac{C}{1-\theta}t^{1-\theta}$ is used to describe the KL condition. We shall induce the $\theta$-KL condition from the our $\beta$-growth condition to show the relationship between constants $C$ and $\Phi$. With $\beta=\frac{1}{1-\theta}$, by rearranging the formula of Assumption \ref{a7}, we have,
\begin{equation}
\begin{aligned}\label{subdiffineq1}
f(\bx;\theta)-f^\star(\theta)&\geq\Phi\mathrm{dist}^\beta(\bx,\,\arg\min_{\bx\in\mathcal{X}}f(\bx,\theta))\\
\Phi^{-1/\beta}(f(\bx;\theta)-f^\star(\theta))^{1/\beta}&\geq\mathrm{dist}(\bx,\,\arg\min_{\bx\in\mathcal{X}}f(\bx,\theta)).
\end{aligned}
\end{equation}
On the other hand, by denoting the error bound function $\phi(t)=\Phi^{-1/\beta} t^{\frac{1}\beta}$, 
\begin{align}\label{subdiffineq2}
\phi(f(\bx;\theta)-f^\star(\theta))\geq \mathrm{dist}(\bx,\,\arg\min_{\bx\in\mathcal{X}}f(\bx,\theta)).
\end{align}
By the Corollary~6.ii of \cite{bolte2017error}, $t\phi'(t)= \frac{1}\beta\phi(t)$, thus,
\[\phi'(f(\bx;\theta)-f^\star(\theta))\Vert \partial_\bx(f(\bx;\theta)+\mathbf{1}_{\mathcal{X}}(\bx))\|_{2-}\geq \frac{1}{\beta}.\]
The $\varphi(t)$ needs to satisfy,
\[\varphi'(f(\bx;\theta)-f^\star(\theta))\Vert \partial_\bx(f(\bx;\theta)+\mathbf{1}_{\mathcal{X}}(\bx))\|_{2-}\geq 1.\]
So we obtain
\[C=\Phi^{-1/\beta}.\]
 Hence, in line with the proof of Theorem 4.iii in \cite{frankel2015splitting}, by knowing $m=\frac{1}{4}$, $C'=\frac{\beta C}{\beta-2}\left(2^{\frac{2-\beta}{2-2\beta}}-1\right)(f(\bx_0;\theta)-f^\star(\theta))^{-\frac{\beta-2}{\beta}}$, and thus $\beta_k=\frac{1}{8 L_{f,2} C^2}$, $c=\min\left\{\frac{C}2,\frac{C'}{\beta_k}\right\}$, we have,
{\small\begin{equation}\begin{aligned}\label{result1}
\frac{(f(\bx_K;\theta)-f^\star(\theta))^{-\frac{\beta-2}{\beta}}}{\Phi^{1/\beta}(\beta-2)}&\geq\min\left\{\frac{\Phi^{1/\beta}}{16\beta L_{f,2}},\,\frac{\left(2^{\frac{2-\beta}{2-2\beta}}-1\right)(f(\bx_0;\theta)-f^\star(\theta))^{-\frac{\beta-2}{\beta}}}{\Phi^{1/\beta}(\beta-2)}\right\}K+\frac{(f(\bx_0;\theta)-f^\star(\theta))^{-\frac{\beta-2}{\beta}}}{\Phi^{1/\beta}(\beta-2)}\\
f(\bx_K;\theta)-f^\star(\theta)&\leq\left(\frac{(\beta-2)\Phi^{2/\beta}}{16\beta L_{f,2}}K+(f(\bx_0;\theta)-f^\star(\theta))^{-\frac{\beta-2}{\beta}}\right)^{-\frac{\beta}{\beta-2}},
\end{aligned}
\end{equation}}
The inequality \eqref{result1} above holds if,
\begin{equation}
\begin{aligned}\label{result2}
\frac{\Phi^{1/\beta}}{16\beta L_{f,2}}&\le\frac{\left(2^{\frac{2-\beta}{2-2\beta}}-1\right)(f(\bx_0;\theta)-f^\star(\theta))^{-\frac{\beta-2}{\beta}}}{\Phi^{1/\beta}(\beta-2)}\\
\frac{\Phi^{2/\beta}(\beta-2)}{16 \beta L_{f,2}\left(2^{\frac{2-\beta}{2-2\beta}}-1\right)}&\le(f(\bx_0;\theta)-f^\star(\theta))^{-\frac{\beta-2}{\beta}}\\
f(\bx_0;\theta)-f^\star(\theta)&\leq \left(\frac{\Phi^{2/\beta}(\beta-2)}{16\beta L_{f,2}\left(2^{\frac{2-\beta}{2-2\beta}}-1\right)}\right)^{-\frac{\beta}{\beta-2}}.
\end{aligned}
\end{equation}
Here we complete the proof.
\end{proof}
Next, we translate the prerequisite of Lemma \ref{lem2} \eqref{result2} into the cover number, then,
\[\overline{M}=\mathbb{M}_{\mathrm{cover}}\left(\Theta,\frac{1}{2L_\Theta}\left(\frac{\Phi^{2/\beta}(\beta-2)}{16\beta L_{f,2}\left(2^{\frac{2-\beta}{2-2\beta}}-1\right)}\right)^{-\frac{\beta}{\beta-2}}\right).\]
If $\beta\in (2,+\infty)$, the convergence rate with respect to the online
oracle cost $K$ is given by $\mathcal{O}\left(K^{-\frac{\beta}{\beta-2}}\right)$.
As $\beta$ approaches $2$ from above, this rate becomes exponentially fast (Proposition \ref{prop:beta_to_2}), indicating that PGD exhibits increasingly rapid
convergence. On the other hand, as $\beta\to\infty$, the convergence rate
asymptotically approaches $\mathcal{O}(1/K)$. 

From Lemma \ref{lem2}, we construct a sufficient condition of $(K,\varepsilon)$-$\Theta$-covering,
\[\left(\frac{(\beta-2)\Phi^{2/\beta}}{16 \beta L_{f,2}}K+(f(\bx_0;\theta)-f^\star(\theta))^{-\frac{\beta-2}{\beta}}\right)^{-\frac{\beta}{\beta-2}}<\varepsilon\]
Then by rearranging terms, we obtain the upper bound on the memory complexity as follows,
\begin{align*}
\Bigg(\frac{(\beta-2)\Phi^{2/\beta}}{16 \beta L_{f,2}}K+\big(f(\bx_0;\theta)-f^\star(\theta)\big)^{-\frac{\beta-2}{\beta}}\Bigg)^{-\frac{\beta}{\beta-2}}&\leq\varepsilon\\
\frac{(\beta-2)\Phi^{2/\beta}}{16\beta L_{f,2}}K+(f(\bx_0;\theta)-f^\star(\theta))^{-\frac{\beta-2}{\beta}}
&\;\ge\; \varepsilon^{-\frac{\beta-2}{\beta}}, \\
(f(\bx_0;\theta)-f^\star(\theta))^{-\frac{\beta-2}{\beta}}
&\;\ge\; \varepsilon^{-\frac{\beta-2}{\beta}}-\frac{(\beta-2)\Phi^{2/\beta}}{16\beta L_{f,2}}K\\
f(\bx_0;\theta)-f^\star(\theta)
&\;\le\; \left(\varepsilon^{-\frac{\beta-2}{\beta}}-\frac{(\beta-2)\Phi^{2/\beta}}{16 \beta L_{f,2}}K\right)^{-\frac{\beta}{\beta-2}}\\
2L_\Theta R_{\Theta}M^{-1/d_\Theta}&\;\lesssim\; \left(\varepsilon^{-\frac{\beta-2}{\beta}}-\frac{(\beta-2)\Phi^{2/\beta}}{16\beta L_{f,2}}K\right)^{-\frac{\beta}{\beta-2}}\\
\left(2L_\Theta R_{\Theta}\left(\varepsilon^{-\frac{\beta-2}{\beta}}-\frac{(\beta-2)\Phi^{2/\beta}}{16\beta L_{f,2}}K\right)^{\frac{\beta}{\beta-2}}\right)^{d_\Theta}&\;\lesssim\; M.
\end{align*} 
As a consequence, by requiring the positiveness, we have,
\begin{align*}M(K,\varepsilon) \;&\leq\;
\mathcal{O}\left(\left[
L_\Theta R_{\Theta}
\left(
\varepsilon^{-\frac{\beta-2}{\beta}}
-
\frac{(\beta-2)\Phi^{2/\beta}}{16\beta L_{f,2}} K
\right)^{\frac{\beta}{\beta-2}}
\right]_+^{d_\Theta}\right).\end{align*}
Hence, we finished the proof of upper bound $\mathcal{O}(\cdot)$. 
\begin{rem}[Prerequisite on $\overline{M}$]
The proof of the upper bound on memory complexity relies on Lemma~\ref{lem2}. Consequently, any bound below $\overline{M}$ is not applicable. Since $\overline{M}$ is a fixed constant determined by $\beta$ and is independent of the online oracle cost $K$ and the accuracy level $\varepsilon$, it can be absorbed into the $\mathcal{O}(\cdot)$ notation.
\end{rem}
\myparagraph{Complexity lower bound} We consider a family of parameterized functions class,
\begin{align}\label{polynomialclassbeta}
\underline{f}(\bx;\theta)
= {\Phi}\left\Vert \bx - \frac{1}{R_\mathcal{X}+R_\Theta}\begin{bmatrix}\theta \\ \mathbf{0}\end{bmatrix} \right\Vert_2^{\beta},
\end{align}
where $\Phi\leq \frac{\min\{L_{f,2},L_\Theta\}}{\beta(\beta-1)(R_\mathcal{X}+R_\Theta)^{\beta-2}}$. This requirement
is not merely a proof artifact, but a hard constraint imposed by the smoothness assumption. In \eqref{polynomialclassbeta}, if $\Phi>\frac{\min\{L_{f,2},L_\Theta\}}{\beta(\beta-1)(R_\mathcal{X}+R_\Theta)^{\beta-2}}$, then $\underline{f}(\bx;\theta)$ will no longer being $L_{f,2}$-smooth. The proof of lower bound of $\beta$-growth is in line with the proof of Theorem \ref{thm3} in Appendix \ref{appendix2}. By directly assigning \eqref{lbthm3}, we have,
\begin{align*}M(K,\varepsilon)&\ge\Omega\left(\left[\left(\frac{\min\{L_{f,2},L_\Theta\}R_\Theta^2}{\varepsilon}\right)^{\frac{\beta-2}{\beta}}-\frac{c_{\mathrm{bound}}\min\{L_{f,2},L_\Theta\}\beta(\beta-2)}{L_{f,2}}K\right]_+^{\frac{d_\Theta}{\beta-2}}\right)\\&=\Omega\left((\min\{L_{f,2},L_\Theta\}R_\Theta^2)^{\frac{d_\Theta}{\beta}}\left[\left(\frac{1}{\varepsilon}\right)^{\frac{\beta-2}{\beta}}-\frac{c_{\mathrm{bound}}\min\{L_{f,2},L_\Theta\}\beta(\beta-2)}{(\min\{L_{f,2},L_\Theta\}R_\Theta^2)^{\frac{\beta-2}{\beta}}L_{f,2}}K\right]_+^{\frac{d_\Theta}{\beta-2}}\right)\end{align*}
Hence, we finished the proof. The lower bound $\Omega(\cdot)$ doesn't have $\Phi$ because, 
\[
f(\bx;\theta) - f^\star(\theta) \ge\frac{\min\{L_{f,2},L_\Theta\}}{\beta(\beta-1)(R_\mathcal{X}+R_\Theta)^{\beta-2}}\mathrm{dist}_2^\beta\!\bigl(\bx,\,\arg\min_{\bx \in \mathcal{X}} f(\bx;\theta)\bigr)\ge \Phi\,\mathrm{dist}_2^\beta\!\bigl(\bx,\,\arg\min_{\bx \in \mathcal{X}} f(\bx;\theta)\bigr).
\]
Hence, the function class with $\Phi_{\mathrm{lb}}=\frac{\min\{L_{f,2},L_\Theta\}}{\beta(\beta-1)(R_\mathcal{X}+R_\Theta)^{\beta-2}}$ also holds for all $\Phi\leq\Phi_{\mathrm{lb}}$.
\end{proof}
\section{Limit behavior as $\beta \to 2^+$}\label{pfprop1}
The third insight in Section \ref{gconvex} can be expressed as the proposition below.
\begin{prop}[Limit behavior as $\beta \to 2^+$]\label{prop:beta_to_2}
Under the same assumptions as in Theorem~\ref{thm2}, as $\beta \to 2^+$, the memory complexity satisfies,

\textbf{(Upper bound).}  If $\Phi \in \left( 0,16L_{f,2}\log 2\right)$, then
\begin{align*}
M(K,\varepsilon)&\leq
\mathcal{O}\!\left(
\left(
\frac{L_{\Theta}R_{\Theta}}{\varepsilon}
\exp\!\left(-\frac{\Phi K}{16L_{f,2}}\right)
\right)^{d_\Theta}
\right).
\end{align*}\textbf{(Lower bound).} If $\Phi \in \left( 0,\frac{\min\{L_{f,2},L_\Theta\}}{\beta(\beta-1) (R_{\mathcal{X}}+R_\Theta)^{\beta-2}}\right)$, then
\begin{align*}
M(K,\varepsilon)&\geq\Omega\!\left(
\left(
\sqrt{\frac{\min\{L_{f,2},L_{\Theta}\}R_\Theta^2}{\varepsilon}}
\exp\!\left(
-\frac{4\min\{L_{f,2},L_{\Theta}\}K}{ L_{f,2}}
\right)
\right)^{d_\Theta}
\right).
\end{align*}
\end{prop}
\begin{proof}
First, we rewrite the Theorem \ref{thm2} as follows.\\
\textbf{(Upper bound).}
\begin{align}\label{betaub}
M(K,\varepsilon)
\leq
\mathcal{O}\!\left(
\left[
\left(\frac{L_{\Theta}R_{\Theta}}{\varepsilon}\right)^{\frac{\beta-2}{\beta}}
- C_1' K
\right]_+^{\frac{\beta d_\Theta}{\beta-2}}
\right).
\end{align}

\textbf{(Lower bound).}
Conversely, if $\Phi \le \frac{\min\{L_{f,2},L_\Theta\}}{\beta(\beta-1) (R_{\mathcal{X}}+R_\Theta)^{\beta-2}}$, then
\begin{align}\label{betalb}
M(K,\varepsilon)
\geq
\Omega\!\left(
\left[
\left(\frac{\min\{L_{f,2},L_{\Theta}\}R_{\Theta}^2}{\varepsilon}\right)^{\frac{\beta-2}{\beta}}
- C_2' K
\right]_+^{\frac{d_\Theta}{\beta-2}}
\right),
\end{align}
where $[\cdot]_+ := \max\{\cdot,0\}$ and if $\beta\to2^+$, recalling the proof of Theorem \ref{thm3}. (Appendix \ref{appendix2}), $c_{\mathrm{bound}}=2$, then,
\[
C_1'
=
\frac{(\beta-2)(L_{\Theta}R_{\Theta})^{\frac{\beta-2}{\beta}}\Phi^{\frac{2}{\beta}}}
{16\beta L_{f,2}},
\qquad
C_2'
=
\frac{2\beta(\beta-2)\min\{L_{f,2},L_\Theta\}}
{L_{f,2}}.
\]
Then, consider the Taylor expansion, we obtain,
\[\left(\frac{L_\Theta R_\Theta}{\varepsilon}\right)^\frac{\beta-2}{\beta}\approx 1+\frac{\beta-2}{2}\log\frac{L_\Theta R_\Theta}{\varepsilon},\,C_1'\approx \frac{(\beta-2)(L_{\Theta}R_{\Theta})^{\frac{\beta-2}{\beta}}\Phi^{\frac{2}\beta}}{32L_{f,2}}. \]
Thus, by applying \eqref{betaub}, 
\begin{align*}
M(K,\varepsilon)
&\leq
\mathcal{O}\!\left(
\lim_{\beta\to 2^+}\left[
\left(\frac{L_{\Theta}R_{\Theta}}{\varepsilon}\right)^{\frac{\beta-2}{\beta}}
- C_1’ K
\right]_+^{\frac{\beta d_\Theta}{\beta-2}}
\right)\\&=\mathcal{O}\!\left(
\lim_{\beta\to 2^+}\left[
1+(\beta-2)\left(\frac{1}{2}\log\frac{L_\Theta R_\Theta}{\varepsilon}-\frac{(L_{\Theta}R_{\Theta})^{\frac{\beta-2}{\beta}}\Phi^{\frac{2}\beta}K}
{32L_{f,2}}\right)
\right]_+^{\frac{\beta d_\Theta}{\beta-2}}
\right)\\&=\mathcal{O}\!\left(\left(
\exp\left(\frac{1}{2}\log\frac{L_\Theta R_\Theta}{\varepsilon}-\frac{\Phi K}
{32L_{f,2}}\right)\right)^{2 d_\Theta}
\right)\\&=\mathcal{O}\!\left(\left(
\left(\frac{L_\Theta R_\Theta}{\varepsilon}\right)\exp\left(-\frac{\Phi K}
{16L_{f,2}}\right)\right)^{ d_\Theta}
\right).
\end{align*}
Conversely, we have,
\[\left(\frac{\min\{L_{f,2},L_{\Theta}\}R_{\Theta}^2}{\varepsilon}\right)^{\frac{\beta-2}{\beta}}\approx 1+\frac{\beta-2}{2}\log\frac{\min\{L_{f,2},L_{\Theta}\}R_\Theta^2}{\varepsilon},\,C_2’
\approx
\frac{4(\beta-2)\min\{L_{f,2},L_\Theta\}}
{L_{f,2}}.\]
by applying \eqref{betalb}, the memory complexity lower bound is,
\begin{align*}
M(K,\varepsilon)
&\geq
\Omega\!\left(
\lim_{\beta\to 2^+}\left[
\left(\frac{\min\{L_{f,2},L_{\Theta}\}R_{\Theta}^2}{\varepsilon}\right)^{\frac{\beta-2}{\beta}}
- C_2’ K
\right]_+^{\frac{d_\Theta}{\beta-2}}
\right)\\&=\Omega\!\left(
\lim_{\beta\to 2^+}\left[
1+(\beta-2)\left(\frac{1}{2}\log\frac{\min\{L_{f,2},L_{\Theta}\} R_\Theta^2}{\varepsilon}-\frac{4\min\{L_{f,2},L_{\Theta}\}K}
{L_{f,2}}\right)
\right]_+^{\frac{ d_\Theta}{\beta-2}}
\right)\\&=\Omega\!\left(\left(
\exp\left(\frac{1}{2}\log\frac{\min\{L_{f,2},L_{\Theta}\}R_\Theta^2}{\varepsilon}-\frac{4\min\{L_{f,2},L_{\Theta}\}K}
{L_{f,2}}\right)\right)^{d_\Theta}
\right)\\&=\Omega\!\left(\left(
\sqrt{\frac{\min\{L_{f,2},L_{\Theta}\} R_\Theta^2}{\varepsilon}}\exp\left(-\frac{4\min\{L_{f,2},L_{\Theta}\}K}
{L_{f,2}}\right)\right)^{ d_\Theta}
\right).
\end{align*}
Next, recall the proof of Theorem \ref{thm2} in Appendix \ref{pfthm2}, the required memory of Lemma \ref{lem2} is,
\[\overline{M}=\mathbb{M}_{\mathrm{cover}}\left(\Theta,\frac{1}{2L_\Theta}\left(\frac{\Phi^{2/\beta}(\beta-2)}{16\beta L_{f,2}\left(2^{\frac{2-\beta}{2-2\beta}}-1\right)}\right)^{\frac{\beta}{\beta-2}}\right).\]
We analyze its behavior when $\beta\to 2^+$. First, we analyze the behavior of $2^{\frac{2-\beta}{2-2\beta}}-1$,
\[2^{\frac{2-\beta}{2-2\beta}}-1=2^{\frac{\beta-2}{2(1+\beta-2)}}-1\approx \frac{\beta-2}{2}\log 2.\]
So, the covering radius is,
\[\lim_{\beta\to2^+}\frac{1}{2L_\Theta}\left(\frac{\Phi^{2/\beta}(\beta-2)}{16\beta L_{f,2}\left(2^{\frac{2-\beta}{2-2\beta}}-1\right)}\right)^{-\frac{\beta}{\beta-2}}\approx\lim_{\beta\to2^+}\frac{1}{2L_\Theta}\left(\frac{\Phi^{2/\beta}}{16L_{f,2}\log 2}\right)^{-\frac{\beta}{\beta-2}}.\]
If $\frac{\Phi}{16L_{f,2}\log 2} < 1$, i.e., $\Phi < 16L_{f,2}\log 2$, then the covering radius $r$ diverges to infinity. Consequently, $\overline{M}\to0$, and Lemma~\ref{lem2} holds unconditionally. Otherwise, if $\Phi > 16L_{f,2}\log 2$, then the covering radius converge to $0$, $\overline{M}\to +\infty$, Lemma \ref{lem2} isn't applicable anymore. Thus the \eqref{betaub} bound is invalid.
\end{proof}
\section{Proof of Corollary \ref{coro1}}\label{pfcoro1}
The proof is straight forward. Recall the linear convergence result again,
\[\min_{k\in[K]}f(\bx_k;\theta)-f^\star(\theta)\leq\rho^{K}(f(\bx_0;\theta)-f^\star(\theta)),\,\forall \bx_0\in\mathcal{X},\,\theta\in\Theta,\]
with the baseline convergence rate $\mathcal{O}\left(\rho^K\right)$ if the initial optimality gap $f(\bx_0;\theta)-f^\star(\theta)$ is considered as a constant. In order to accelerate the convergence speed to $\mathcal{O}((\alpha\rho)^K)$. We need to let $f(\bx_0;\theta)-f^\star(\theta)\sim\alpha^K$. It implies we need to let,
\[\|\pi_\mathcal{D}(\theta)-\theta\|\sim\frac{\alpha^K}{2L_\Theta}.\]
Hence, we need,
\[M_{\alpha}(K)\leq\mathcal{O}\left(\left(R_\Theta L_\Theta\left(\frac{1}{\alpha}\right)^K\right)^{d_\Theta}\right).\]
\section{Proof of Corollary \ref{coro3}}\label{pfcoro3}
Similar to above, according to Lemma \ref{lem2}, the baseline convergence rate is,
\begin{align*}
\min_{k\in[K]}f(\bx_k;\theta)-f^\star(\theta) \;\le\;
\left(\frac{(\beta-2)\Phi^{2/\beta}}{16\beta L_{f,2}}K+(f(\bx_0;\theta)-f^\star(\theta))^{-\frac{\beta-2}{\beta}}\right)^{-\frac{\beta}{\beta-2}},\,\forall \bx_0\in\mathcal{X},\,\theta\in\Theta,
\end{align*}
with rate $\mathcal{O}\left(K^{-\frac{\beta}{\beta-2}}\right)$ if the initial optimality gap $f(\bx_0;\theta)-f^\star(\theta)$ is considered as a constant. If we want it to have a $\alpha$-acceleration, we need to let,
\[\left(f(\bx_0;\theta)-f^\star(\theta)\right)^{-\frac{\beta-2}{\beta}}\sim K^{\frac{1}{\alpha}}.\]
The target is to let the initial optimality gap outperforms the first term $\frac{(\beta-2)\Phi^{2/\beta}}{16\beta L_{f,2}}K\sim K$ if $\beta>2$ is fixed. So, we have to let $f(\bx_0;\theta)-f^\star(\theta)\sim K^{-\frac{\beta}{\alpha(\beta-2)}}$, thus,
\[\|\pi_\mathcal{D}(\theta)-\theta\|\sim\frac{K^{-\frac{\beta}{\alpha(\beta-2)}}}{2L_\Theta}.\]
Hence, we need,
\[
M_\alpha(K)\leq\mathcal{O}\left(\left(R_{\Theta}L_{\Theta}K^\frac{\beta}{\alpha(\beta-2)}\right)^{d_\Theta}\right).
\]
\section{Acceleration when $\beta\to2^+$}\label{pfspeed}
\paragraph{Discussion.} As $\beta \to 2^+$, the polynomial sample complexity in Corollary~\ref{coro3} no longer provides any effective acceleration. This is because the leading term $\frac{(\beta-2)\Phi^{2/\beta}}{16 \beta L_{f,2}}K$ in \eqref{klgua} from Lemma~\ref{lem2} is no longer of order $\mathcal{O}(K)$, but instead vanishes as $\beta \to 2^+$. From Lemma~\ref{lem2}, if $\Phi < 16L_{f,2}\log 2$, the baseline convergence rate of PGD reduces to
\begin{align}\label{pgd2}
\min_{k\in[K]} f(\bx_k;\theta) - f^\star(\theta)
\le
\left(f(\mathbf{x}_0;\theta) - f^\star(\theta)\right)
\exp\!\left(-\frac{\Phi}{16L_{f,2}} K\right),
\end{align}
which exhibits linear convergence. To further improve the rate to $\mathcal{O}\left( \left(\alpha e^{-\frac{\Phi}{16L_{f,2}}}\right)^K\right)$ as in Corollary~\ref{coro1}, the required memory scales as $M_{\alpha}(K)
\leq
\mathcal{O}\!\left(
\left(
R_{\Theta} L_\Theta
\left(
\frac{1}{\alpha}
\right)^K
\right)^{d_\Theta}
\right)$.
\paragraph{Proof of \eqref{pgd2}.} Recall \eqref{klgua} in Lemma \ref{lem2} again, 
\begin{align*}
\min_{k\in[K]}f(\bx_k;\theta)-f^\star(\theta) \;\le\;
\left(\frac{(\beta-2)\Phi^{2/\beta}}{16\beta L_{f,2}}K+(f(\bx_0;\theta)-f^\star(\theta))^{-\frac{\beta-2}{\beta}}\right)^{-\frac{\beta}{\beta-2}},\,\forall \bx_0\in\mathcal{X},\,\theta\in\Theta.
\end{align*}
So,
\begin{align*}&\lim_{\beta\to 2^+}\left(\frac{(\beta-2)\Phi^{2/\beta}}{16\beta L_{f,2}}K+(f(\bx_0;\theta)-f^\star(\theta))^{-\frac{\beta-2}{\beta}}\right)^{-\frac{\beta}{\beta-2}}\\=&\lim_{\beta\to 2^+}\left(\frac{(\beta-2)\Phi^{2/\beta}}{16\beta L_{f,2}}K-\frac{\beta-2}2\log(f(\bx_0;\theta)-f^\star(\theta))+1\right)^{-\frac{\beta}{\beta-2}}\\=&\exp\left(-\frac{\Phi K}{16L_{f,2}}+{\log(f(\bx_0;\theta)-f^\star(\theta))}\right)\\=&(f(\bx_0;\theta)-f^\star(\theta))\exp\left(-\frac{\Phi K}{16L_{f,2}}\right).
\end{align*}
\section{How to derive the convergence guarantee of $\beta$-growth from \cite{li2018calculus}}
The original result in \cite{li2018calculus} is stated for the KL condition with exponent $\theta \in [0,1)$. By applying the first conclusion in Corollary~6 of \cite{bolte2017error} to the augmented convex, lower-semi continuous function $f(\bx) + \mathbf{1}_{\mathcal{X}}(\bx)$, one obtains the corresponding $\beta$-growth property with exponent $\beta = \frac{1}{1-\theta}$. This establishes the claimed convergence order under the $\beta$-growth condition.
\section{Numerical experiment (full version)}\label{fullexperiment}
The matrix $A$ used for  Tikhonov regression (2-decimal places),
\setcounter{MaxMatrixCols}{12}
{\small{\[
A =
\begin{bmatrix}
-0.10 & -0.02 &  0.14 &  0.04 & -0.22 &  0.54 &  0.29 &  0.38 &  0.42 & -0.27 &  0.33 & -0.20 \\
 0.33 & -0.37 &  0.13 &  0.27 & -0.04 &  0.16 & -0.13 & -0.12 &  0.36 &  0.53 & -0.03 &  0.05 \\
-0.25 & -0.22 & -0.17 &  0.35 & -0.14 &  0.01 & -0.04 &  0.20 & -0.06 & -0.13 & -0.05 &  0.54 \\
-0.16 &  0.23 & -0.01 & -0.22 & -0.19 & -0.05 & -0.16 & -0.01 &  0.42 &  0.02 & -0.33 &  0.13 \\
-0.28 &  0.10 &  0.19 &  0.11 & -0.26 &  0.05 & -0.23 & -0.13 & -0.15 &  0.14 &  0.16 & -0.11 \\
-0.14 & -0.22 &  0.18 & -0.02 &  0.17 & -0.05 & -0.06 & -0.22 &  0.13 & -0.22 &  0.03 &  0.00
\end{bmatrix}.
\]}}
\end{document}